\documentclass{article}
\usepackage{graphicx} % Required for inserting images
\usepackage{xcolor}
\usepackage{amsmath}
\usepackage{amssymb}
\usepackage{amsthm}
\usepackage{authblk}
\usepackage{geometry}
\usepackage{subcaption}  % To use subfigure environment
\usepackage{colortbl} % for coloring cells in tables
\usepackage{xcolor} % for color specifications
\usepackage{tikz}  % to adjust transparency of cell colors
\usepackage{booktabs}
\usepackage{siunitx}
\usepackage{orcidlink}
\usepackage{multirow}
\usepackage{url}
\usepackage[ruled, noend]{algorithm2e}
\usepackage[left]{lineno}
\theoremstyle{plain}
\newtheorem{proposition}{Proposition}[section]

\theoremstyle{remark}

\usepackage[
  backend=biber,
  style=numeric,
  sorting=none,
  giveninits=true,
  maxnames=5,
  minnames=1,
  doi=true,
  url=true,
  isbn=false
]{biblatex}
\DeclareNameAlias{author}{family-given}
\DeclareNameAlias{editor}{family-given}
\DeclareDelimFormat{finalnamedelim}{\addspace\&\space}

\title{Learning partially observed systems with neural Hamiltonian ordinary differential equations}

\author[1, 2, *]{Sunniva Meltzer\,\orcidlink{0009-0003-5012-9718}}
\author[1]{Sølve Eidnes\,\orcidlink{0000-0002-1002-3543}}
\author[1, 3]{Alexander Johannes Stasik\,\orcidlink{0000-0003-1646-2472}}
\affil[1]{Department of Mathematics and Cybernetics, SINTEF Digital, Norway}
\affil[2]{Department of Physics, University of Oslo, Norway}
\affil[3]{Department of Data Science, Norwegian University of Life Science, Norway}
\affil[*]{Corresponding author: sunniva.meltzer@sintef.no}

\date{}

\begin{document}
%\linenumbers
\maketitle

\begin{abstract}

When learning dynamical systems from data, embedding physical structure can constrain the solution space and improve generalization, but many physics-informed models assume access to the full system state. This limits their use in partially observed settings, where some state variables are completely unobserved and must be inferred without direct supervision. Here, we present neural Hamiltonian ordinary differential equations (NHODE), a framework that combines Hamiltonian neural networks (HNNs) with neural ordinary differential equations (neural ODEs) to learn partially observed dynamical systems from data. The Hamiltonian structure enforces energy conservation by construction, while the neural ODE framework enables a flexible training procedure that allows the loss to be defined only on observed variables. We also incorporate additional physical constraints through symmetry-aware coordinate transformations and separable energy formulations. The framework is evaluated on systems of increasing complexity, from linear and nonlinear mass-spring systems to the chaotic three-body problem. Across all examples, increasing the amount of embedded physical structure improves the accuracy and long-horizon stability of the predictions. Even in the most challenging regimes, the NHODE framework captures both observed and latent dynamics, whereas purely data-driven baselines become unstable.

\end{abstract}

%\section{Introduction}

\noindent
In many physical systems of practical interest, some state variables are inaccessible to measurement, yet their dynamics shape the evolution of the observed state. To accurately model such systems, the corresponding region of the phase space must be treated as fully latent, forcing the model to infer the system dynamics there without receiving direct feedback during training as in a classical machine learning framework. If the latent dynamics and their coupling to the observed system are learned incorrectly, the model will not generalize to unseen data.

In this work, we show that by incorporating physical knowledge into a machine learning model, we can accurately learn not only the observed but also the latent part of a partially observed system, where purely data-driven methods fail. To incorporate physics, we build on Hamiltonian neural networks (HNNs) \cite{Greydanus2019HamiltonianNetworks}, and to facilitate learning of full system dynamics from partial data, we build on neural ordinary differential equations (ODEs) \cite{Chen2018NeuralEquations}. Specifically, HNNs enforce energy conservation by construction, while rollout-based neural ODE training allows the model to evolve a full latent state while evaluating the loss only on the observed components. We call the resulting framework \textit{neural Hamiltonian ODEs} (NHODE), and use it to model a set of increasingly complex physical systems. We identify when and to what extent physical information embedded in the model renders a problem learnable.

Throughout this paper, we use the term \textit{neural ODE} to describe a learned continuous-time dynamical-system model intended to represent a physical system, rather than in the continuous-depth neural network sense in which an ODE is used to describe the evolution of hidden units over an artificial time variable. This distinction has been a point of confusion in recent literature. The latter definition is the one made explicit in the paper by Chen et al.\ \cite{Chen2018NeuralEquations} that popularized neural ODEs, and it is also the perspective underlying preceding related work by Weinan E and Haber and Ruthotto \cite{E2017ASystems, Haber2017StableNetworks}. By contrast, our use of the term neural ODE is the modelling sense where a neural network parameterizes the vector field of an actual system to be learned from data, as in the earlier system-identification approaches of Chu and Shoureshi \cite{Chu1991ASystems} and Rico-Martínez et al.\ \cite{Rico-Martinez1992Discrete-Data}. In this sense, our work is closer to the physical-modelling viewpoint encompassed by Kidger's broader framework of neural differential equations \cite{Kidger2021OnEquations}, where neural ODEs are treated both as continuous-depth architectures and as tools for modelling real dynamical systems, including settings with latent or partially observed state variables. Other works following this approach include \cite{Botev2021WhichDynamics, Quaglino2020SNODE:Identification, Rahman2022NeuralIdentification,Matsubara2023FINDE:Quantities}. Rackaukas and collaborators, who introduced the related concept of universal differential equations \cite{Rackauckas2021UniversalLearning}, have extended this view of neural ODEs beyond deterministic finite-dimensional systems to partial differential equations \cite{Rackauckas2019DiffEqFlux.jlEquations} and Bayesian inference frameworks \cite{Dandekar2022BayesianEquations}.

Greydanus et al.\ introduced Hamiltonian neural networks (HNNs) \cite{Greydanus2019HamiltonianNetworks} (see also the contemporary work \cite{Bertalan2019OnData}), in which a neural network model approximating the Hamiltonian function is learned to fit the corresponding canonical Hamiltonian system to data.  The loss is defined on the residual of Hamilton's equations, requiring data on all state variables and their time derivatives, or finite-difference approximations thereof, at every time step. Much effort has been put into making the training more efficient or accurate by utilizing symplectic \cite{Chen2020SymplecticNetworks, David2023SymplecticNetworks} or symmetric and higher-order integration schemes \cite{Eidnes2023Pseudo-HamiltonianForces, Celledoni2025LearningIntegrators}, or through post-training corrections of the learned vector field based on backward error analysis \cite{Zhu2020DeepIntegrators, Offen2022SymplecticSystems}. However, if data on all state variables is not available, the training approach assumed in these works is not viable, since the input to the neural network must be data representing all state variables at any time step.

To that end, the crucial ingredient that we borrow from the neural ODE framework is that numerical integration is performed at every step of the training loop. This is as opposed to the standard approach for HNNs, where the loss is defined on the residual of the ODE or a given integration scheme. In neural ODE training, the model is integrated forward from an initial condition to produce a predicted trajectory, and the loss is evaluated between the predicted and observed state at the data time points. Importantly, this roll-out technique means that the neural network modelling the Hamiltonian takes predictions of the state, and not data, as input, except at the initial step. Thus, we can meaningfully restrict the loss to the observed components of the state alone. 

Among existing works, the Hamiltonian generative network approach of Toth et al.\ \cite{Toth2020HamiltonianNetworks} is closely related to our neural Hamiltonian ODEs in that it combines a learned Hamiltonian with rollout-based training under indirect observations, albeit in a generative modelling setting and with latent variables from pixels. On the other hand, the method of Grigorian et al.\ \cite{Grigorian2025LearningSystems} is closely related to our work from the point of view of partial observability: it uses a hybrid neural ODE framework to learn governing equations of partially observed systems with unobserved state variables, training on the observed variables alone and subsequently using symbolic regression to recover explicit equations. Related to this is also the recent physics-informed neural-ODE-based approaches to system identification considered in \cite{Buisson-Fenet2023RecognitionODEs, Ghanem2024LearningMeasurements, Malani2023SomeInformation}. Also relevant in this setting are earlier works combining a Hamiltonian structure with a neural ODE training approach \cite{Zhong2024SymplecticControl, Sanchez-Gonzalez2019HamiltonianIntegrators}, although they only consider fully observed systems.

\section{Results}
\label{sec:results}
%For each model: short description, motivation, true Hamiltonian, system dimensionality, characteristics, observed variables, and generation of initial conditions.

Our neural Hamiltonian ODE method is designed for system identification and prediction of partially observed Hamiltonian systems. More specifically, we consider Hamiltonian systems
\begin{equation}
    \dot{\mathbf{x}} = S \nabla \mathcal{H}(\mathbf{x}),
    \label{eq:hamilton}
\end{equation}
where $\mathbf{x} = (\mathbf{q}, \mathbf{p})^{\mathsf{T}} \in \mathbb{R}^{2n}$ denotes the generalized position and momentum coordinates of a system with $n$ degrees of freedom, $\mathcal{H} : \mathbb{R}^{2n} \rightarrow \mathbb{R}$ is the Hamiltonian, and %$S \in \mathbb{R}^{2n} \times \mathbb{R}^{2n}$ is the canonical symplectic matrix 
% \begin{equation*}
% S =
% \begin{pmatrix}
% 0 & I \\
% -I & 0
% \end{pmatrix}.
% \end{equation*}
$S = \left(\begin{smallmatrix}
    0 & I \\
    -I & 0
\end{smallmatrix}\right) \in \mathbb{R}^{2n \times 2n}$ is the canonical symplectic matrix.
For a system of $N_p$ point masses in $d$ spatial dimensions, $n = N_p \, d$. The Hamiltonian is a conserved quantity of the system and typically represents the total energy. We define partially observed systems as systems where one or more of the $2d$ variables are completely unobserved at all times. %For most examples, we assume access to the initial condition of the full state space. However, in section \ref{sec:learning-initial-conditions}, we demonstrate that learning is also possible without full knowledge of the initial state. 

The general training procedure is presented in Figure \ref{fig:partially-observed-method}.
\begin{figure}[!b]
    \centering
    \includegraphics[width=\linewidth]{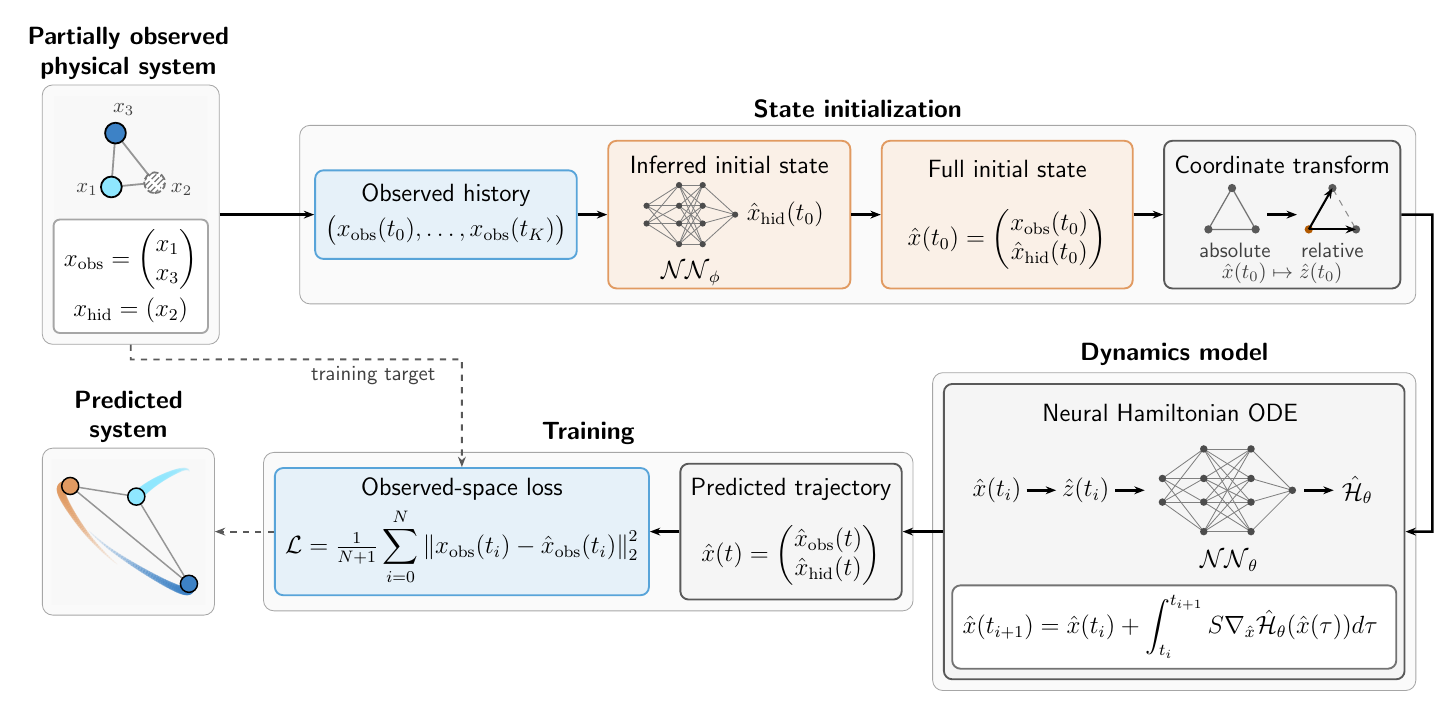}
    \caption{Training pipeline for neural Hamiltonian ODEs. From an observed history, a neural network infers the initial state of unobserved components, which is combined with the initial state of observed variables to form the full system state. The state is optionally transformed to symmetry-aware coordinates (e.g., relative distances) before being propagated through a neural Hamiltonian ODE to generate a trajectory. Training is performed by minimizing a loss function defined only on observed components of the predicted trajectory.}
    \label{fig:partially-observed-method} 
\end{figure}
The partially observed system is illustrated by three point masses (particles) connected by three springs, where we only observe two of the point masses. The first $K$ time steps of the observed trajectories $(\mathbf{x}_1(t), \mathbf{x}_3(t))^{\mathsf{T}}$ are used to learn the initial condition $\mathbf{x}_2(t_0)$ of the hidden point mass. When translational symmetry is assumed, we encode this prior by transforming position coordinates to pairwise distances and providing these as model inputs with the momentum coordinates. This yields a translation-invariant learned Hamiltonian, ensuring conservation of linear momentum. The neural Hamiltonian ODE model is then given by transforming the state vectors $\hat{x}(t_i)$ to appropriate coordinates $\hat{z}(t_i)$, depending on the priors, which are passed to a neural network $\mathcal{H}_\theta$ that approximates the Hamiltonian. The equations of motion are obtained by taking the gradient of $\mathcal{H}_\theta$ with respect to $\hat{x}(t_i)$, and the state at the next point in time, $\hat{x}(t_{i+1})$, is found through numerical integration with an ODE solver of choice. After unrolling full trajectories from each initial state in the training data set, the loss is evaluated across the observed dimensions only, so that $\hat{\mathbf{x}}_{\textrm{hid}}$ is left out of the computational graph. The influence of $\hat{\mathbf{x}}_{\textrm{hid}}$ on the training is therefore only indirect, via the evolution of the observed particles over time.

For all the presented examples, we assume that both the position and momentum are unobserved for one of $d$ particles, with $d$ being two or three. Initially, for all examples, we assume to know the initial state of all coordinates, avoiding inferring the initial hidden state, while Section \ref{sec:learning-initial-conditions} considers the complication of also learning the initial hidden state. Throughout Section \ref{sec:results}, a clear relationship is observed between performance and the amount of domain knowledge incorporated into the model. The various levels of physics-informed training are explained in more detail in Section \ref{sec:incorporating-physics}. The most important distinction is that we have trained two different types of neural Hamiltonian ODEs for all examples, referred to as NHODE$_{\textrm{tot}}$ and NHODE$_{\textrm{pot}}$ in the following. The former is a model set up to learn the total energy of the system, while the latter is a more physics-informed model where only the potential energy is learned by a neural network. This is done by assuming the Hamiltonian is separable, $\mathcal{H}(\mathbf{q}, \mathbf{p}) = T(\mathbf{p}) + V_\theta(\mathbf{q})$, and assuming the kinetic energy term is on the form $T(\mathbf{p}) = \mathbf{p}^2/2m$. The latter implies that we also assume the masses $m$ of the particles are known. For the NHODE$_{\textrm{pot}}$ models where the position coordinates are transformed to relative pairwise distances, the resulting Hamiltonian is rotation-invariant in addition to translation-invariant, implying conservation of both linear and angular momentum. An overview of the physical quantities conserved by each model type is provided in Table \ref{tab:overview_phys_quantities_conserved_by_models}. For a proof of these conserved properties, see Appendix \ref{sec:proofs}. 

\begin{table}[tbh]
    \centering
    \begin{tabular}{cccc}
        \toprule
         System & Total energy & Linear momentum & Angular momentum \\
         \midrule
         NHODE$_{\textrm{tot,abs}}$ & Yes & Not guaranteed & Not guaranteed \\
         NHODE$_{\textrm{tot,rel}}$& Yes & Yes & Not guaranteed \\
         NHODE$_{\textrm{pot,abs}}$ & Yes & Not guaranteed & Not guaranteed \\
         NHODE$_{\textrm{pot,rel}}$ & Yes & Yes & Yes \\
         \bottomrule
    \end{tabular}
    \caption{Overview of the physical quantities conserved by each model. The NHODE$_{\textrm{tot,abs}}$ and NHODE$_{\textrm{pot,abs}}$ models are trained on absolute coordinates, while NHODE$_{\textrm{pot,rel}}$ and NHODE$_{\textrm{tot,rel}}$ are trained on pairwise distances between particles where translational symmetry is assumed.}
    \label{tab:overview_phys_quantities_conserved_by_models}
\end{table}

In all examples, we compare the neural Hamiltonian ODEs to two different baseline models. The first is a standard neural ODE model, referred to as NODE$_{\textrm{vanilla}}$ in the figures and tables. The second is a physics-informed variant of neural ODEs where only the momentum derivatives are learned by the model, referred to as NODE$_{\textrm{phys}}$. The motivation behind this is to investigate if physical priors lead to better performance also in the NODE framework. While NODE$_{\textrm{vanilla}}$ is a good baseline for NHODE$_{\textrm{tot}}$, the NODE$_{\textrm{phys}}$ model is a fairer comparison for the NHODE$_{\textrm{pot}}$ model.

The following sections present three different physical examples with increasing dimensionality and complexity; see Figure \ref{fig:all-example-schematics}. For all examples, and for each of the four different methods, we have trained 10 models using different random initializations. If not otherwise stated, the presented results are the mean value of the predictions from these. Across all examples, we observe that increasing the amount of embedded physical structure systematically improves both the accuracy and long-horizon stability of the learned dynamics, and is essential for learning in the more complex and chaotic settings.

\begin{figure}[tbh]
    \centering
    \includegraphics[width=\linewidth]{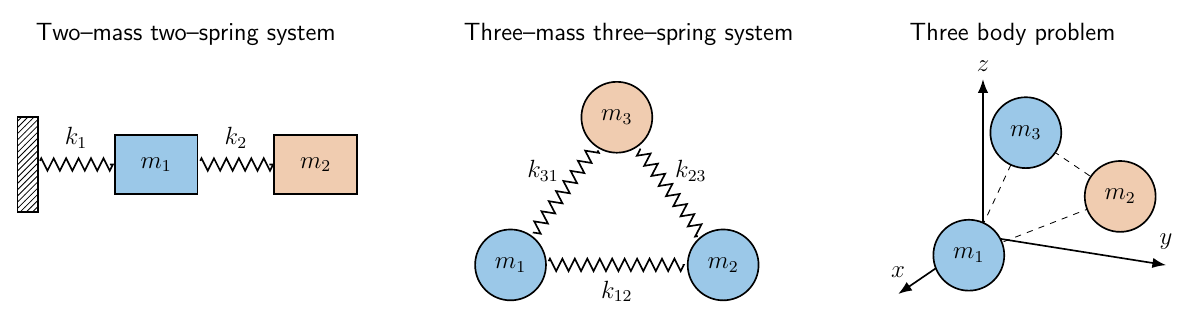}
    \caption{Dynamical systems investigated in this study. From left to right, we have a linear one dimensional two--mass two--spring system, a nonlinear three--mass three--spring system in two dimensions, and the chaotic three body problem in three dimensions. The blue point masses are observed, while the orange point masses are unobserved during training.}
    \label{fig:all-example-schematics}
\end{figure}

\subsection{Two--mass two--spring system}

The first example we consider is the two--mass two--spring system in one dimension. The first spring connects a stationary wall to the first point mass, and the second spring connects the first point mass to the second point mass, as illustrated in Figure \ref{fig:all-example-schematics}. Both springs are linear. The true Hamiltonian for this system is
\begin{equation}
\mathcal{H}(\mathbf{q},\mathbf{p})
= 
\underbrace{\sum_{i=1}^{2}\frac{p_i^2}{2m_i}}_{\textrm{kinetic energy } T}
+ 
\underbrace{\sum_{j=1}^{2}\frac{k_j}{2}\,e_j^2(\mathbf{q}),}_{\textrm{potential energy } V}
\label{eq:two-linear-mass-spring-hamiltonian}
\end{equation}
where
\begin{equation*}
\begin{aligned}
e_1(\mathbf{q}) &= |q_1| - \ell_1,\\
e_2(\mathbf{q}) &= |q_2 - q_1| - \ell_2.
\end{aligned}
\label{eq:two-mass-spring-extensions}
\end{equation*}
Here, $q_i$ and $p_i$ denote the generalized position and momentum of particle $i$, respectively, and $k_j$ and $l_j$ are the spring constant and rest length of spring $j$, respectively. For this system, we assume that the second particle is unobserved. The initial conditions across the trajectories in the training and test datasets were generated randomly from uniform distributions, with $x_1(0) \sim \mathcal{U}(0.2, 0.8)$, $x_2(0) \sim \mathcal{U}(0.9, 1.4)$, and $v_1(0), v_2(0) \sim \mathcal{U}(-0.7, 0.7)$, where $v_i$ is the velocity of particle $i$. The masses were set to $m_1 = 1.0$ and $m_2 = 1.2$, the spring constants were set to $k_1 = 3.0$ and $k_2 = 5.0$, and the rest lengths of the springs were set to $\ell_1 = 0.4$, $\ell_2 =  0.6$.

The predicted positions and momenta over time for a random test initial condition of the two point masses are presented in Figure \ref{fig:two-mass-spring-pred-trajectories-and-energy-error}, as well as the mean squared error (MSE) for the observed and unobserved point mass and the absolute error (AE) of the predicted Hamiltonian.

\begin{figure}[tbh]
    \centering
    \includegraphics[width=\linewidth]{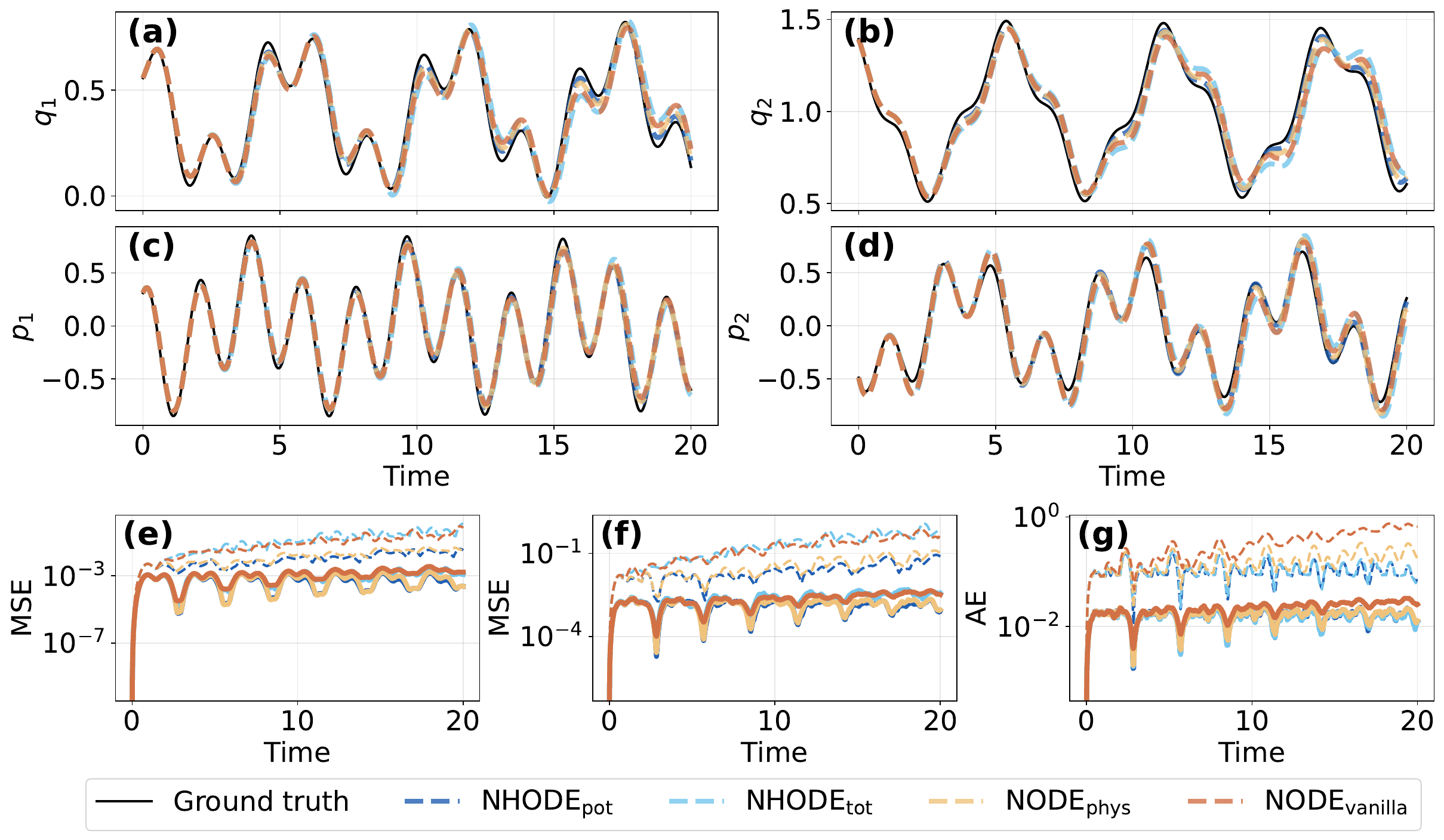}
    \caption{Predictions for the linear two-mass two-spring system. Panel (a) shows the predicted position of the observed point mass, panel (b) the predicted position of the unobserved point mass, panel (c) the predicted momentum of the observed point mass, and panel (d) the predicted momentum of the unobserved point mass. Panel (e) and (f) shows the median of the mean squared error for the observed and unobserved point mass, respectively, and (g) shows the median of the absolute total energy error. The latter is calculated by evaluating the true Hamiltonian on the predicted trajectories. The dashed lines in (e--g) indicate the maximum errors.}
    \label{fig:two-mass-spring-pred-trajectories-and-energy-error}
\end{figure}

\subsection{Triangular nonlinear mass--spring system}

Next, we consider a two-dimensional three--mass three--spring system where the three point masses are connected in a triangular manner, as illustrated in Figure \ref{fig:all-example-schematics}. All three springs are nonlinear, and chosen such that the force on each particle is proportional to the spring extension cubed. The true Hamiltonian therefore scales with the fourth power of the spring extension:
\begin{equation}
\mathcal{H}(\mathbf{r},\mathbf{p})
=
\underbrace{\sum_{i=1}^{3}\frac{\|\mathbf{p}_i\|^2}{2m_i}}_{\textrm{kinetic energy } T}
\;+\;
\underbrace{\sum_{(i,j)\in\mathcal{E}} \frac{k_{ij}}{4}\,e_{ij}(\mathbf{r})^{4}}_{\textrm{potential energy } V},
\label{eq:three-nonlinear-mass-spring-hamiltonian-compact}
\end{equation}
where
\begin{equation*}
\begin{aligned}
e_{ij}(\mathbf{r}) &:= \|\mathbf{r}_j-\mathbf{r}_i\| - L_{ij},
\quad (i,j)\in\mathcal{E}.
\end{aligned}
\label{eq:three-mass-spring-extension}
\end{equation*}
Here, $\mathbf{r}_i=(x_i,y_i)$ and $\mathbf{p}_i=(p_{x_i},p_{y_i})$ denote the position and momentum of point mass $i$, respectively, $L_{ij}$ is the rest length of the spring connecting point masses $i$ and $j$, and $\mathcal{E}=\{(1,2),(1,3),(2,3)\}$ is the edge set of the triangular spring network.

The initial conditions were generated randomly from uniform distributions, with $x_1(0) \sim \mathcal{U}(0.3, 0.7)$, $y_1(0) \sim \mathcal{U}(1.0, 1.4)$, $x_2(0) \sim \mathcal{U}(0.1, 0.4)$, $y_2(0) \sim \mathcal{U}(0.1, 0.5)$, $x_3(0) \sim \mathcal{U}(0.6, 0.9)$, $y_3(0) \sim \mathcal{U}(0.1, 0.5)$ and $v_1(0), v_2(0), v_3(0) \sim \mathcal{U}(-0.7, 0.7)$. The initial velocities were then adjusted to ensure zero center-of-mass momentum. The spring constants were set to $k_1 = k_2 = k_3 = 1.0$, and the rest lengths of the springs $L_{12} = L_{13} = L_{23} = 0.9$. All masses $m_i$ were set to 1. Since this system is translational invariant, all models were trained on relative pairwise distances.

\begin{figure}[bh]
    \centering
    \includegraphics[width=\linewidth]{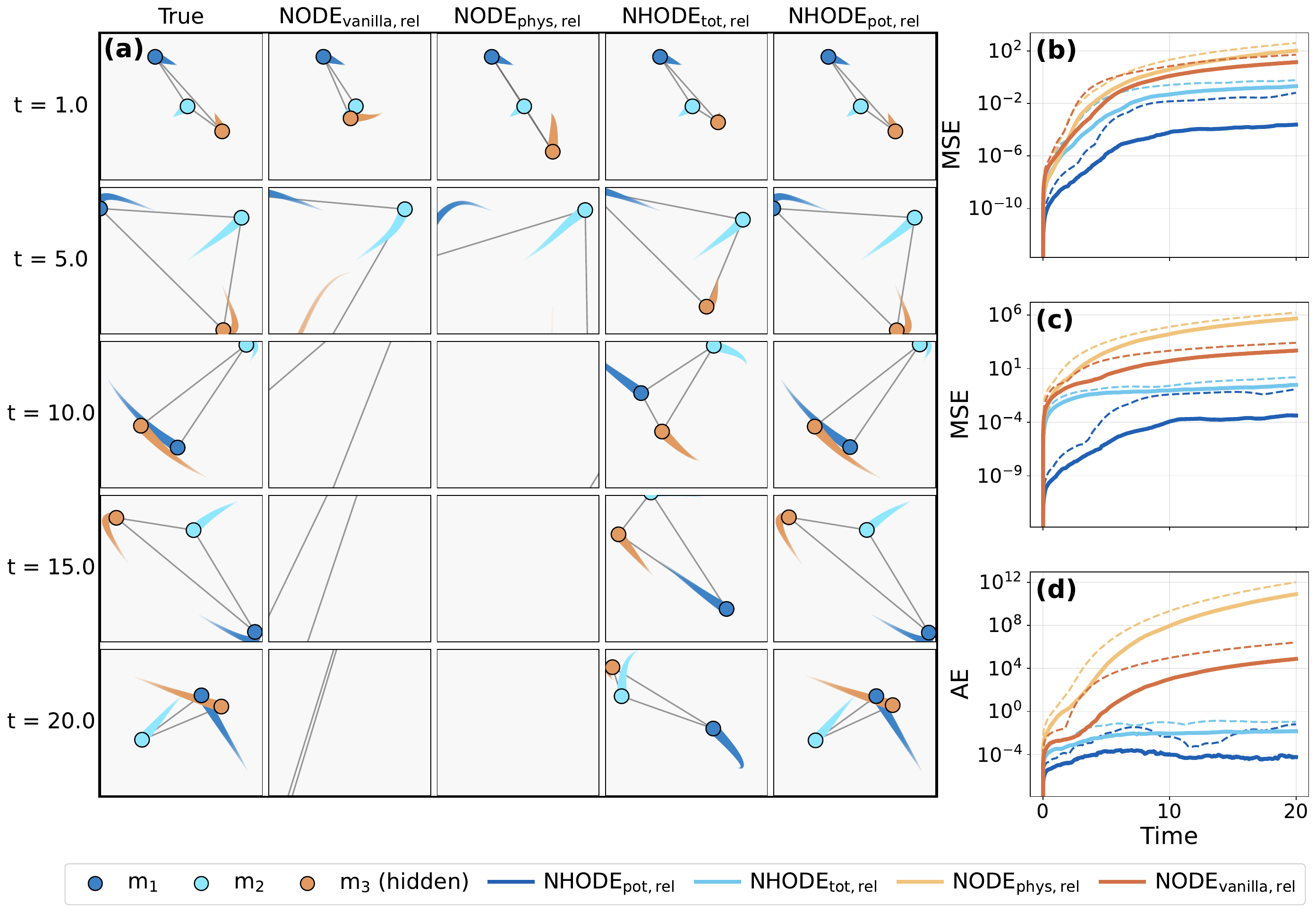}
    \caption{Model prediction for the triangular nonlinear mass-spring system. Panel (a) shows a grid of snapshots, where each column contains snapshots of the system at 5 different time steps. The left column shows the ground truth, while column 2-5 gives the mean predictions from 
    the models NODE$_\textrm{vanilla,rel}$, NODE$_\textrm{phys,rel}$, NHODE$_\textrm{tot,rel}$ and NHODE$_\textrm{pot,rel}$, respectively. Panel (b) shows the median MSE of the observed point masses, panel (c) the median MSE of the unobserved point mass, and panel (d) shows the median absolute total energy error. The dashed lines in (b--d) indicate the maximum errors.}
    \label{fig:three-nonlinear-mass-spring-combined-snapshot-and-errors}
\end{figure}

The predicted system configurations at various time points starting from a random test initial condition are presented in Figure \ref{fig:three-nonlinear-mass-spring-combined-snapshot-and-errors}(a). Figure \ref{fig:three-nonlinear-mass-spring-combined-snapshot-and-errors} (b--d) shows the prediction errors over time, presented as MSE on the observed point masses, MSE on the unobserved point mass, and absolute error of the total energy, i.e.\ Hamiltonian. The exact errors are also given in Table \ref{tab:triangular-nonlinear-mass-spring-error-table}, where we compare the median errors on unseen initial states rolled out over an interval corresponding to the length of the trajectories in the training dataset and a longer testing timespan.

\begin{table}[tbh]
\centering
\caption{Median prediction errors for the triangular nonlinear three--mass three-spring system over short- and long-horizon rollouts. The short interval corresponds to the length of the training window, while the long interval is 20 times longer.}
\begin{tabular}{lcccc}
\toprule
& \multicolumn{2}{c}{$t \in [0,1]$} & \multicolumn{2}{c}{$t \in [0,20]$} \\
\cmidrule(lr){2-3} \cmidrule(lr){4-5}
Method & Observed & Unobserved & Observed & Unobserved \\
\midrule
NHODE$_{\textrm{pot,rel}}$ & $9.40\mathit{e-}10$ & $8.00\mathit{e-}10$ & $2.44\mathit{e-}04$ & $4.24\mathit{e-}04$ \\
NHODE$_{\textrm{tot,rel}}$    & $2.35\mathit{e-}07$ & $2.85\mathit{e-}03$ & $2.06\mathit{e-}01$ & $3.00\mathit{e-}01$ \\
NODE$_{\textrm{phys,rel}}$   & $2.46\mathit{e-}07$ & $9.84\mathit{e-}02$ & $1.04\mathit{e+}02$ & $4.60\mathit{e+}05$ \\
NODE$_{\textrm{vanilla,rel}}$         & $9.35\mathit{e-}07$ & $3.92\mathit{e-}02$ & $1.38\mathit{e+}01$ & $4.74\mathit{e+}02$ \\
\bottomrule
\end{tabular}
\label{tab:triangular-nonlinear-mass-spring-error-table}
\end{table}

\subsection{Three-body problem}

We consider the three-body problem, the motion of three point masses interacting via mutual gravitational attraction, as a final example to demonstrate the full benefit of our method. As a frictionless classical mechanical system, it conserves energy, linear momentum, and angular momentum, while exhibiting chaotic and dynamically rich behaviour that makes long-horizon prediction particularly challenging. The Hamiltonian is given by
\begin{equation}
\mathcal{H}(\mathbf r,\mathbf p)
=
\underbrace{\sum_{i=1}^{3}\frac{\|\mathbf p_i\|^{2}}{2m_i}}_{\text{kinetic energy }T}
\;-\;
\underbrace{\left(
\sum_{1\le i<j\le 3}
\frac{G\,m_i m_j}
{\sqrt{\|\mathbf r_i-\mathbf r_j\|^{2}+\varepsilon^{2}}}
\right)}_{\text{potential energy }V}
\end{equation}
where $\mathbf r_i$ and $\mathbf p_i$ denote the position and momentum of body $i$, $m_i$ is its mass and  $G$ is the gravitational constant. The potential is regularized by using a Plummer softening length $\varepsilon$ \cite{Aarseth1963DynamicalGalaxies, Dehnen2001TowardsError}, which improves the stability of the numerical simulations by smoothing the potential at small inter-body distances.
In our setup, we let the masses be equal, providing a symmetric and non-trivial test problem. We use relatively large values of $\varepsilon$ to improve numerical stability, while not changing the fundamental qualitative properties of the problem. 

As for the other examples, the initial conditions were picked randomly, with $x_1(0) \sim \mathcal{U}(0.3, 0.7)$, $y_1(0) \sim \mathcal{U}(1.0, 1.4)$, $z_1(0) \sim \mathcal{U}(-0.1, 0.1)$, $x_2(0) \sim \mathcal{U}(0.1, 0.4)$, $y_2(0) \sim \mathcal{U}(0.1, 0.5)$, $z_2(0) \sim \mathcal{U}(-0.1, 0.1)$, $x_3(0) \sim \mathcal{U}(0.6, 0.9)$, $y_3(0) \sim \mathcal{U}(0.1, 0.5)$, $z_3(0) \sim \mathcal{U}(-0.1, 0.1)$ and $v_1(0), v_2(0), v_3(0) \sim \mathcal{U}(-0.3, 0.3)$. The initial positions were subsequently adjusted to recenter the center of mass to origin, and the initial velocities were adjusted to ensure zero total momentum. The gravitational constant was set to $G=1$, and $m_1 = m_2 = m_3 = 1$. We selected $\varepsilon = 0.6$ based on an epsilon range test to determine the optimal value of the Plummer softening parameter.
The test was performed by training 10 models on different values of $\varepsilon$, and comparing the average test error. The results from this test are provided in Appendix \ref{sec:plummer-parameter-test}.

Figure \ref{fig:three-body-combined-snapshot-and-errors} shows the predicted trajectories from a random initial condition for the four different methods; NHODE$_{\textrm{pot,rel}}$, NHODE$_{\textrm{tot,rel}}$, NODE$_{\textrm{phys,abs}}$ and NODE$_{\textrm{vanilla,abs}}$. Note that the NHODE models here are trained on relative coordinates while the neural ODE models were trained on absolute coordinates, as we found that this gave the best results within each framework. The prediction errors are presented both graphically in Figure \ref{fig:three-body-combined-snapshot-and-errors} and with numbers in Table \ref{tab:three-body-problem-error-table}, where we have included results for NHODE trained on absolute positions and NODE trained on relative coordinates as well.

\begin{figure}[h!]
    \centering
    \includegraphics[width=\linewidth]{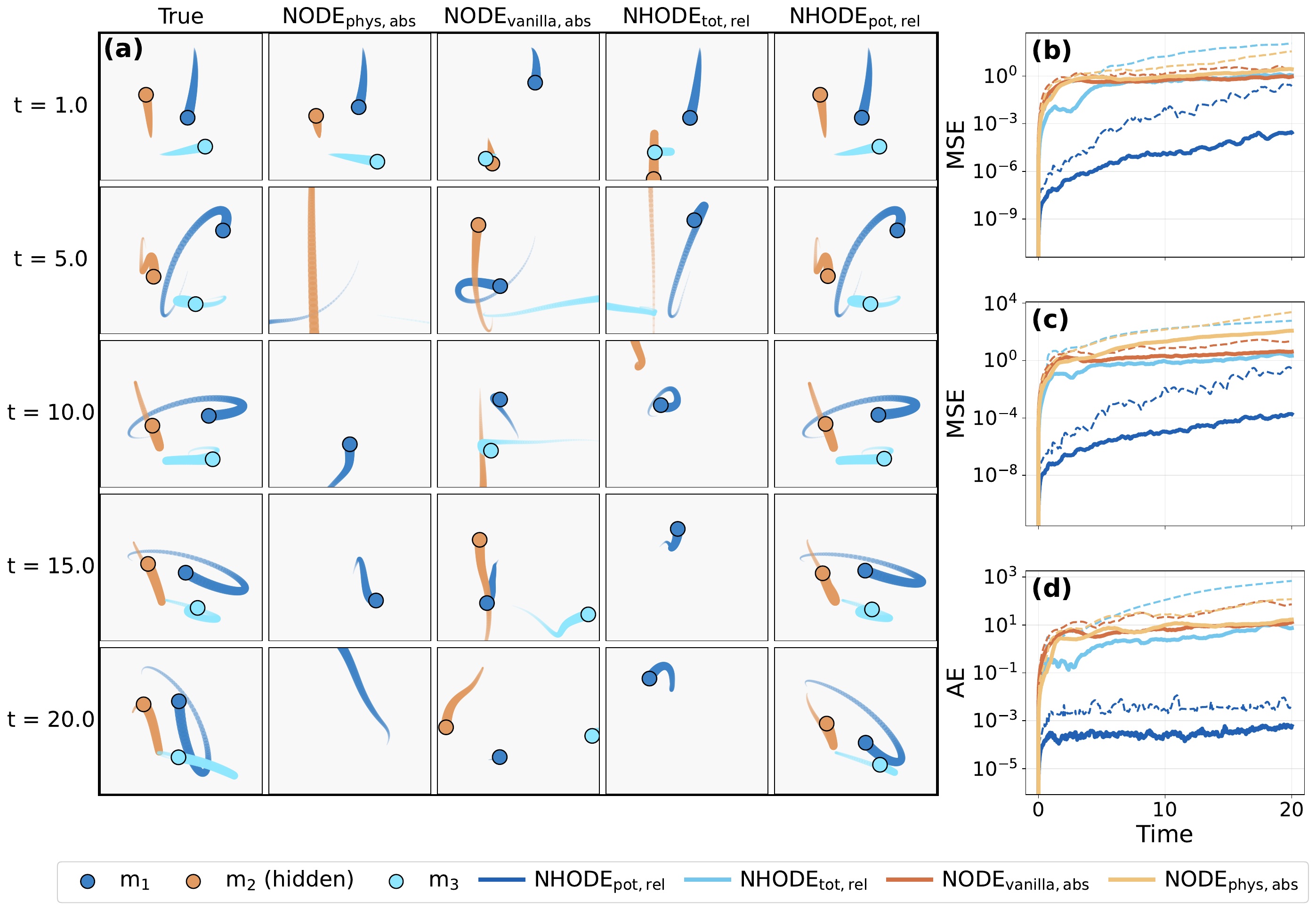}
    \caption{Model prediction for the three body problem. Panel (a) shows a grid of snapshots, where each column contains snapshots of the system at 5 different time steps. The left column shows the ground truth, while column 2-5 gives the mean predictions from the models NODE$_\textrm{vanilla,abs}$, NODE$_\textrm{phys,abs}$, NHODE$_\textrm{tot,rel}$ and NHODE$_\textrm{pot,rel}$, respectively. Panel (b) shows the median MSE of the observed point masses, panel (c) the median MSE of the unobserved point mass, and panel (d) shows the median absolute total energy error. The dashed lines in (b--d) indicate the maximum errors.}
    \label{fig:three-body-combined-snapshot-and-errors}
\end{figure}

\begin{table}[tbh]
\centering
\caption{Prediction errors for the three body problem over short- and long-horizon rollouts, including comparison of the methodologies when trained on absolute and relative coordinates.}
\begin{tabular}{lcccc}
\toprule
& \multicolumn{2}{c}{$t \in [0,1]$} & \multicolumn{2}{c}{$t \in [0,20]$} \\
\cmidrule(lr){2-3} \cmidrule(lr){4-5}
Method & Observed & Unobserved & Observed & Unobserved \\
\midrule
NHODE$_{\textrm{pot,rel}}$    & $5.67\mathit{e-}08$ & $5.96\mathit{e-}08$ & $2.76\mathit{e-}04$ & $1.87\mathit{e-}04$ \\
NHODE$_{\textrm{pot,abs}}$    & $3.96\mathit{e-}02$ & $4.45\mathit{e-}02$ & $7.00\mathit{e-}01$ & $1.55\mathit{e+}01$ \\
NHODE$_{\textrm{tot,rel}}$    & $6.58\mathit{e-}03$ & $5.97\mathit{e-}02$ & $1.12\mathit{e+}00$ & $2.16\mathit{e+}00$ \\
NHODE$_{\textrm{tot,abs}}$    & $2.12\mathit{e-}01$ & $1.14\mathit{e+}00$ & $2.68\mathit{e+}01$ & $7.50\mathit{e+}01$ \\
NODE$_{\textrm{phys,rel}}$    & $1.87\mathit{e-}03$ & $2.56\mathit{e+}00$ & $1.86\mathit{e+}03$ & $8.44\mathit{e+}03$ \\
NODE$_{\textrm{phys,abs}}$    & $3.85\mathit{e-}02$ & $1.13\mathit{e-}01$ & $2.73\mathit{e+}00$ & $1.19\mathit{e+}02$ \\
NODE$_{\textrm{vanilla,rel}}$ & $2.20\mathit{e-}03$ & $1.09\mathit{e+}00$ & $4.29\mathit{e+}00$ & $8.30\mathit{e+}01$ \\
NODE$_{\textrm{vanilla,abs}}$ & $1.11\mathit{e-}01$ & $4.24\mathit{e-}01$ & $9.36\mathit{e-}01$ & $4.10\mathit{e+}00$ \\
\bottomrule
\end{tabular}
\label{tab:three-body-problem-error-table}
\end{table}

\colorlet{grey}{white!95!black} 

\definecolor{my_red}{HTML}{C2482A} 
\colorlet{my_red}{my_red!30!white}

\definecolor{my_orange3}{HTML}{D27146}
\colorlet{my_orange3}{my_orange3!30!white}

\definecolor{my_orange2}{HTML}{E19A61}
\colorlet{my_orange2}{my_orange2!30!white}

\definecolor{my_orange1}{HTML}{F0C37C}
\colorlet{my_orange1}{my_orange1!30!white}

\definecolor{my_yellow}{HTML}{FFEC97}
\colorlet{my_yellow}{my_yellow!30!white}

\definecolor{my_blue1}{HTML}{8FE7FF} 
\colorlet{my_blue1}{my_blue1!30!white}

\definecolor{my_blue2}{HTML}{74C6EC} 
\colorlet{my_blue2}{my_blue2!30!white}

\definecolor{my_blue3}{HTML}{58A4D9} 
\colorlet{my_blue3}{my_blue3!30!white}

\definecolor{my_blue4}{HTML}{3D82C6} 
\colorlet{my_blue4}{my_blue4!30!white}

\definecolor{my_blue5}{HTML}{2160B3} 
\colorlet{my_blue5}{my_blue5!30!white}

\newcommand{\shadecell}[2]{%
  \tikz[baseline=(X.base)]{
    \node[fill=#1, fill opacity=.5, text opacity=1, inner sep=2pt] (X) {#2};
  }%
}

The results for all three examples considered in the previous subsections are summarized in Table \ref{tab:result-summary-table-with-colors}. The colors reflect the performance of each method when applied to each system, ranging from dark blue (low error) to orange/red (high error). For the two--mass two--spring system, all four methods are trained on absolute positions, while for the nonlinear three--mass three-spring system, translational symmetry was assumed, so all four methods are trained on relative coordinates. In the three body problem, the NHODE models are trained on relative coordinates and the NODE models on absolute positions, as this gave the best performance.

\begin{table}[tbh]
    \centering
    \caption{Comparison of method performance
 across all investigated systems. The errors in the table are the median of the mean squared errors for both the observed and unobserved dimensions after time $t=20$, corresponding to $2000$ rollout steps with $dt=0.01$. The color scale reflects error magnitude, transitioning from dark blue (low error) to orange/red (high error).}
    \begin{tabular}{lcccc}
         \toprule
         System & NODE$_{\textrm{vanilla}}$ & NODE$_{\textrm{phys}}$ & NHODE$_{\textrm{tot}}$ & NHODE$_{\textrm{pot}}$ \\
         \midrule
         Two-linear mass--spring & \cellcolor{my_blue2} $3.11\mathit{e-}03$ & \cellcolor{my_blue3} $7.28\mathit{e-}04$ & \cellcolor{my_blue2}  $2.82\mathit{e-}03$ & \cellcolor{my_blue3}  $6.32\mathit{e-}04$ \\
         Three-nonlinear mass--spring & \cellcolor{my_orange2} $1.66\mathit{e+}02$ &\cellcolor{my_red} $1.54\mathit{e+}05$& \cellcolor{my_blue1} $2.61\mathit{e-}01$& \cellcolor{my_blue4} $3.16\mathit{e-}04$\\
         Three-body problem & \cellcolor{my_yellow} $2.06\mathit{e+}00$ & \cellcolor{my_orange1} $4.37\mathit{e+}01$ & \cellcolor{my_yellow} $1.39\mathit{e+}00$ & \cellcolor{my_blue5} $2.75\mathit{e-}04$ \\
         \bottomrule
    \end{tabular}
    \label{tab:result-summary-table-with-colors}
\end{table}

\noindent

\subsection{Learning initial conditions}
\label{sec:learning-initial-conditions}

To demonstrate the initial-state inference step of the neural Hamiltonian ODE framework, we consider again the two--mass two--spring system \eqref{eq:two-linear-mass-spring-hamiltonian}. We use a fully-connected neural network as an encoder to infer the initial conditions of the hidden point mass, given the first $K$ steps of the trajectory of the observed point mass. We train the encoder and the Hamiltonian neural network simultaneously using the same global loss function. This is explained in more detail in Section \ref{sec:neuralODE-training-and-learning-IC}.

Figure \ref{fig:two-mass-spring-learning-ic-pred-trajectories-and-energy-error} presents the predicted evolution over time of both positions and both momenta, from a random test initial condition. Figure \ref{fig:two-mass-spring-learning-ic-pred-trajectories-and-energy-error} also shows the median mean squared errors over time of the observed and unobserved particles, as well as the median of the absolute energy error, computed by evaluating the true Hamiltonian on the predicted trajectories.

\begin{figure}[hbt]
    \centering
    \includegraphics[width=\linewidth]{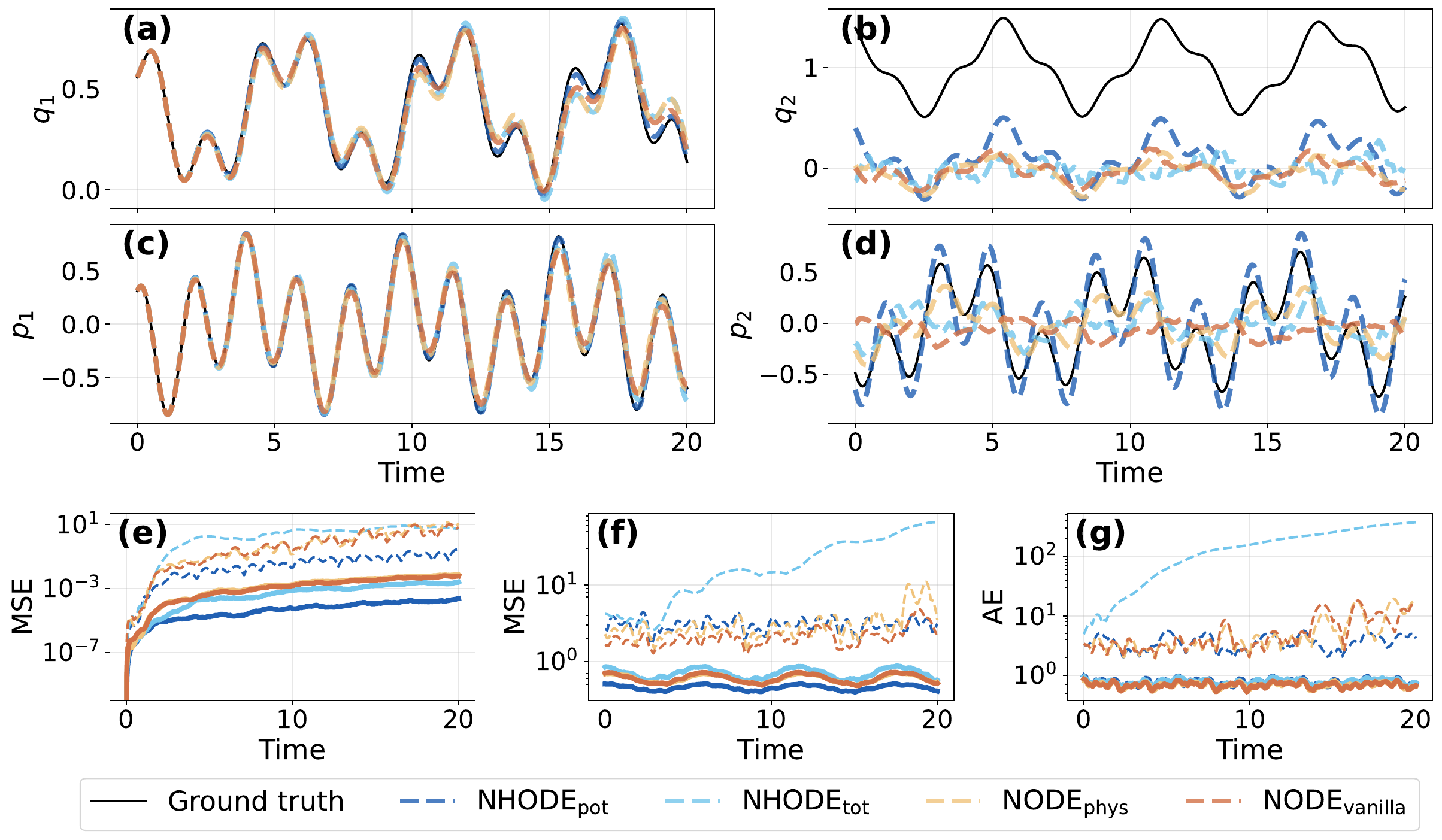}
    \caption{Model predictions for the linear two--mass two--spring system when initial conditions of the hidden point mass are learned. Panel (a) and (c) shows the predicted position and momentum of the point mass that was observed during training, respectively, and panel (b) and (d) shows the predicted position and momentum of the hidden point mass, respectively, where the initial condition is inferred before the trajectory is rolled out. Panel (e) and (f) shows the median of the mean squared error for the observed and unobserved point mass, respectively, and (g) shows the median of the absolute total energy error. The latter is calculated by evaluating the true Hamiltonian on the predicted trajectories. The dashed lines in (e--g) indicate the maximum errors.}
    \label{fig:two-mass-spring-learning-ic-pred-trajectories-and-energy-error}
\end{figure}

\section{Discussion}

One of the central motivations behind physics-informed machine learning is that physics-informed models can compensate for missing information in the available data and succeed where purely data-driven models will fail \cite{Karniadakis2021Physics-informedLearning}. However, few works in this emerging field demonstrate this as clearly as we do here. We consider the case where one particle in the system is completely unobserved, a common situation in real-world settings. This is a fundamentally more challenging task than having sparse or unevenly sampled data or data with noise, but where the data still cover all degrees of freedom, which are the more typical examples of imperfect data considered in the literature. While computationally demanding, our proposed framework is generalizable to more complex and larger systems, and to the case of lacking observations on more than one particle. 

Across the three example systems studied, a clear pattern emerges. In the linear one-dimensional mass--spring system, the differences between the NHODE models and the data-driven baselines are small, and all models succeed in learning the system dynamics. However, as complexity increases and the dynamics become more challenging, this gap widens substantially. Notably, for the chaotic three-body problem, the only model variant that consistently produce non-diverging trajectories over long-horizon rollouts is the most physics-informed one, where the learned Hamiltonian, linear and angular momentum are all conserved. 
Since a key characteristic of chaotic systems is that minor perturbations in initial conditions, model parameters or numerical solver will lead to exponentially diverging trajectories over time, there is a limit to how long any method practically could match a reference trajectory. Given this limitation, the long-term accuracy and stability observed for the NHODE$_\textrm{pot,rel}$ model suggest that incorporating stronger physical priors could extend the window over which predictions remain useful. For longer time horizons where the predicted trajectories are expected to diverge from the true system, our results also indicate that adding physical constraints prevents the learned dynamics from blowing up, which is a common problem when learning chaotic systems (see Appendix \ref{sec:super-long-rollouts-three-body-problem}). 

A key challenge when learning systems where some components are completely unobserved is that the latent trajectories might not be identifiable from the observed dynamics alone.  
Although physical priors compensates for the lacking information by shrinking the solution space of the model, there could in general be more than one unique solution for the latent variables that matches the observed coordinates. Our results on initial condition learning for the hidden variables suggests that the NHODE model in such cases learns a latent state that is physically consistent and gives accurate predictions for the observed variables, but this learned reconstruction is not necessarily equal to the true hidden state. A simple derivation for the linear two--mass two--spring system shows that unless the rest length of the spring connecting the two masses is known, there is no unique solution for the latent dynamics: We construct an additional Hamiltonian $\tilde{\mathcal{H}}(\tilde{\mathbf q},\mathbf p)$, where $\tilde{\mathbf q} = (q_1, \tilde{q}_2)$ and $\tilde{q}_2 = q_2 + a$, and require that 
\begin{equation}
    \nabla_{\mathbf q}\mathcal{H}(\mathbf q,\mathbf p)
    =
    \nabla_{\tilde{\mathbf q}}\tilde{\mathcal{H}}(\tilde{\mathbf q},\mathbf p).
    \label{eq:uniqueness-by-equal-derivatives-of-hamiltonian}
\end{equation}
In other words, we change the position of the hidden point mass by a distance $a$ and require that the two Hamiltonians still give the same equations of motion for the observed point mass. For the new system governed by $\tilde{\mathcal{H}}$, we also change the rest length and spring constant of the spring connecting the two point masses to $\tilde{l}_2$ and $\tilde{k}_2$, respectively.
Since only the potential energy term in \eqref{eq:two-linear-mass-spring-hamiltonian} depends on positions, \eqref{eq:uniqueness-by-equal-derivatives-of-hamiltonian} is equivalent to solving 
$\nabla_{\mathbf q}\mathcal{V}(\mathbf q) 
= 
\nabla_{\tilde{\mathbf q}}\tilde{\mathcal{V}}(\tilde{\mathbf q})$,
where
\begin{equation*}
    \begin{aligned}
        \mathcal{V}(q_1, q_2) &= \frac{k_1}{2}(|q_1| - l_1)^2 + \frac{k_2}{2}(|q_2 - q_1| - l_2)^2 \\
        \tilde{\mathcal{V}}(q_1, \tilde{q}_2) &= \frac{k_1}{2}(|q_1| - l_1)^2 + \frac{\tilde{k}_2}{2}(|\underbrace{q_2 + a}_{\tilde{q}_2} - q_1| - \tilde{l}_2)^2.
    \end{aligned}
\end{equation*}
Solving for $a$ gives the following expression, 
\begin{equation}
    a = \left( \frac{k_2}{\tilde{k}_2} - 1\right)(q_2 - q_1) +
    \begin{cases} 
        \frac{k_2}{\tilde{k}_2}l_2 + \tilde{l}_2 & \text{if } (q_2 - q_1) < 0 \text{ and } (q_2 + a - q_1) > 0, \\
        \frac{k_2}{\tilde{k}_2}l_2 - \tilde{l}_2 & \text{if } (q_2 - q_1) < 0 \text{ and } (q_2 + a - q_1) < 0,\\
        - \frac{k_2}{\tilde{k}_2}l_2 + \tilde{l}_2 & \text{if } (q_2 - q_1) > 0 \text{ and } (q_2 + a - q_1) > 0,\\
        - \frac{k_2}{\tilde{k}_2}l_2 - \tilde{l}_2 & \text{if } (q_2 - q_1) > 0 \text{ and } (q_2 + a - q_1) < 0,\\
    \end{cases}
    \label{eq:no-unique-solution-learning-ic-double-mass-spring}
\end{equation}
meaning that a non-zero shift in the initial position of the second point mass can be compensated by modifying the rest length and/or spring constant, while inducing the exact same dynamics for the observed point mass.
In general, recovering the true latent dynamics would require additional information, e.g.\ known system parameters or occasional measurements of the hidden components. This is also supported by our results, which indicate that when the full initial state is known, the NHODE model learns both the observed and latent dynamics with high accuracy.

Although the NHODE method builds on HNNs and concerns only the class of canonical Hamiltonian systems, a similar approach can easily be extended to other physics structures. Canonical Hamiltonian systems can be generalised to invariant-preserving systems and further to port- or pseudo-Hamiltonian systems, and the NHODE method could be generalised accordingly, building on corresponding extensions of HNN \cite{Desai2021Port-HamiltonianSystems, Eidnes2023Pseudo-HamiltonianForces}. The Lagrangian formalism, and the corresponding method called Lagrangian neural networks (LNN) \cite{Cranmer2020LagrangianNetworks}, would naturally make way for a method called neural Lagrangian ODEs (NLODE), with a similar training pipeline as that of NHODE. Similarly, extensions of HNNs and LNNs that incorporate graph structures for interacting systems \cite{Sanchez-Gonzalez2019HamiltonianIntegrators, Cranmer2020DiscoveringBiases, Cranmer2020LagrangianNetworks} would integrate cleanly with our proposed methodology. Furthermore, soft-constraining approaches like physics-informed neural networks (PINNs) \cite{Lagaris1998ArtificialEquations, Raissi2019Physics-informedEquations} could also be applied, either in addition to or instead of the hard-constrained physics of NHODE. In the latter case, this would mean a standard neural ODE method combined with our initial state decoder and an additional physics term in the loss function. Moreover, the typical data challenges of noise and sparse observations have not been considered in this paper, but the NHODE framework is set up to handle these through the neural ODE training approach. We expect that models with stronger physical priors would be better equipped to handle such data limitations.

From an application perspective, the framework demonstrated by the NHODE method is relevant in settings where i) a dynamical system is only partially observed, ii) some kind of physical structure is known, and iii) the goal is to make stable predictions over time. Examples include industrial monitoring where only a subset of sensors is available or reliable, mechanical systems in robotics where some internal states are not measured, and scientific experiments where some degrees of freedom cannot be directly observed but are coupled to measured quantities through known conservation laws. In such scenarios, the primary objective is often to obtain stable and physically consistent predictions for the observed components, while the details of the latent dynamics might be of secondary importance.

\section{Methods}
\label{sec:methods}

We model our dynamical systems using Hamiltonian neural networks, combined with a training procedure inspired by neural ODEs. The central idea is to embed physical structure through the Hamiltonian framework, while allowing some components of the state remain completely unobserved during training. Additional physical constraints are added through symmetry-aware coordinate transformations and a separable energy formulation. The general training procedure is described in Algorithm \ref{alg:training}. In our implementation, the numerical integration step 
\begin{equation}
\mathrm{ODESolve}(f_\theta,\hat{\mathbf{x}}_i,t_i,t_{i+1}) \approx \hat{\mathbf{x}}_i + \int_{t_i}^{t_{i+1}}f_\theta(\hat{\mathbf{x}}(\tau)) \, d\tau
\end{equation}
was performed using a fifth-order Runge--Kutta solver with adaptive step size control~\cite{Tsitouras2011RungeKuttaAssumption} from Kidger's diffrax-library~\cite{Kidger2021OnEquations}. The parameter update OptimStep was performed using the Adam optimizer~\cite{Kingma2017Adam:Optimization}, with $s$ holding the running first- and second-moment estimates of $\nabla_\theta \mathcal{L}$ or $\nabla_\phi \mathcal{L}$. See Appendix \ref{sec:experimental-details} for more experimental details. Here and in the following, we have used the notation $\hat{\mathbf{x}}_i = \hat{\mathbf{x}}(t _i)$ for the predicted state at time step $t _i$.

\begin{algorithm}[H]
\caption{Training a neural Hamiltonian ordinary differential equation}
\KwIn{Observed trajectory $\mathbf{X}_{\mathrm{obs}} = \{\mathbf{x}_{\mathrm{obs}}(t_i)\}_{i=0}^N$, set of observed indices $\mathcal I_{\mathrm{obs}}$, number of observed time steps $K$ to infer the hidden initial condition}
\KwOut{Trained model parameters $\theta$ and, if used, encoder parameters $\phi$}

Initialize parameters $\theta$ (and $\phi$, if used) and the corresponding optimizer states $s_\theta$ (and $s_\phi$) \;
\BlankLine
\For{each training step}{
    \eIf{full initial state is known}{
        $\hat{\mathbf{x}}_0 \leftarrow \mathbf{x}(t_0)$
    }{
        $\hat{\mathbf{x}}_{\mathrm{hid},0} \leftarrow \mathcal{NN}_\phi(\mathbf{x}_{\mathrm{obs}}(t_{0:K}))$ \\
        $\hat{\mathbf{x}}_0 \leftarrow \left(\mathbf{x}_{\mathrm{obs}}(t_{0}),\, \hat{\mathbf{x}}_{\mathrm{hid,0}}\right)$
    }
    \For{$i = 0,\dots,N-1$}{
        Separate $\hat{\mathbf{x}}_i \rightarrow (\hat{\mathbf{q}}_i, \hat{\mathbf{p}}_i)$ \;
        \eIf{translational symmetry is assumed}{
            $\tilde{\mathbf{q}}_i \leftarrow d(\hat{\mathbf{q}}_i)$
        }{
            $\tilde{\mathbf{q}}_i \leftarrow \hat{\mathbf{q}}_i$
        }
        \eIf{separable Hamiltonian with known kinetic energy $\mathcal{T}$ is assumed}{
            $\hat{\mathcal{H}}_\theta \leftarrow \mathcal{T}(\hat{\mathbf{p}}_i) + \mathcal{NN}_\theta(\tilde{\mathbf{q}}_i)$ \;
        }{
            $\hat{\mathcal{H}}_\theta \leftarrow \mathcal{NN}_\theta(\tilde{\mathbf{q}}_i, \hat{\mathbf{p}}_i)$ \;
        }
        $f_\theta \leftarrow S\nabla_{\hat{\mathbf{x}}_i} \hat{\mathcal{H}}_\theta$ \;
        $\hat{\mathbf{x}}_{i+1} \leftarrow \mathrm{ODESolve}(f_\theta, \hat{\mathbf{x}}_i, t_i, t_{i+1})$ \;
    }
    $\mathcal L \leftarrow \frac{1}{N+1}\sum_{i=0}^{N}
    \left\|\hat{\mathbf{x}}_{\mathrm{obs},i}-\mathbf{x}_{\mathrm{obs}}(t_i)\right\|_2^2$ \;
    $(\theta, s_\theta) \leftarrow \mathrm{OptimStep}(\theta, \nabla_\theta \mathcal L, s_\theta)$ \;
    \If{encoder is used}{
        $(\phi, s_\phi) \leftarrow \mathrm{OptimStep}(\phi, \nabla_\phi \mathcal L, s_\phi)$ \;
    }
}
\label{alg:training}
\end{algorithm}

\subsection{Hamiltonian dynamics and Hamiltonian neural networks}

Canonical Hamiltonian systems are described by Hamilton's equations
\begin{equation}
    \dot{\mathbf{q}} = \nabla_{\mathbf{p}} \mathcal{H}(\mathbf{q},\mathbf{p}),
    \qquad
    \dot{\mathbf{p}} = -\nabla_{\mathbf{q}} \mathcal{H}(\mathbf{q},\mathbf{p}),
    \label{eq:Hamiltons_equations}
\end{equation}
where $\mathbf{q}, \mathbf{p} \in \mathbb{R}^n$ are the generalized position and momentum coordinates, and the Hamiltonian $\mathcal{H} : \mathbb{R}^n \times \mathbb{R}^n \rightarrow \mathbb{R}$ represents the generalised total energy and is conserved over time. The system \eqref{eq:Hamiltons_equations} is often expressed in the condensed form \eqref{eq:hamilton}, where $\mathbf{x} = (\mathbf{q}, \mathbf{p})^{\mathsf{T}}$ and $\mathcal{H} : \mathbb{R}^{2n} \rightarrow \mathbb{R}$.

A Hamiltonian neural network is a network that learns the mapping from $\mathbf{q}$ and $\mathbf{p}$ to the Hamiltonian function, 
\begin{equation*}
    \mathcal{H}(\mathbf{q}, \mathbf{p}) \approx \hat{\mathcal{H}}_{\theta}(\mathbf{q}, \mathbf{p}) = \mathcal{NN}_\theta(\mathbf{q}, \mathbf{p}).
\end{equation*}

\noindent
From the learned $\mathcal{H}_{\theta}$, the equations of motion are obtained by assuming Hamilton's equations (\ref{eq:Hamiltons_equations}) and using automatic differentiation, i.e.\ backpropagation, to obtain $\nabla\hat{\mathcal{H}}_\theta(\mathbf{x})$, which guarantees conservation of the approximated total energy $\hat{\mathcal{H}}_\theta$ for the dynamical system implied by the model.

When training Hamiltonian neural networks \cite{Greydanus2019HamiltonianNetworks}, the loss is normally evaluated as
\begin{equation*}
\mathcal{L} = \frac{1}{N+1}\sum_{i=0}^{N} \left\lVert \dot{\mathbf{x}}(t_i) - S \nabla \hat{\mathcal{H}}_\theta(\mathbf{x}(t_i)) \right\rVert_2^2,
\end{equation*}
or if data on time derivatives are not available, on a numerical integration scheme,
\begin{equation}
    \mathcal{L} = \frac{1}{N}\sum_{i=0}^{N-1}  \left\lVert \mathbf{x} (t_{i+1}) - \mathrm{ODESolve}(S \nabla \hat{\mathcal{H}}_\theta, \mathbf{x}(t_i), t_i, t_{i+1}) \right\rVert_2^2,
    \label{eq:hnn_training}
\end{equation}
where ODESolve here can be any mono-implicit integrator that relies on data from two successive time points \cite{Celledoni2025LearningIntegrators}; e.g., the midpoint method
\begin{equation*}
\mathrm{ODESolve}(S \nabla \hat{\mathcal{H}}_\theta, \mathbf{x}(t_i), t_i, t_{i+1}) = \mathbf{x}(t_i) + S \nabla \hat{\mathcal{H}}_\theta \bigg(\frac{\mathbf{x}(t_i)+\mathbf{x}(t_{i+1})}{2}\bigg).
\end{equation*}
However, both these approaches require that training data on all variables in the system are available, which is often not the case in real-world scenarios. In our method, we circumvent this problem by employing a loss formulation and training approach similar to what is used in the neural ODE literature.

Here and in the following, the loss is written for a single trajectory, and does not include batch dimensions and spatial or state-component dimensions for notational simplicity. In the implementation, the mean squared error is computed over all corresponding observed scalar entries.

\subsection{Neural ODE training and learning the initial state}
\label{sec:neuralODE-training-and-learning-IC}
For general neural ODEs, the derivative of the state vectors, corresponding to right-hand side of the general first-order ODE $\dot{\mathbf{x}} = f(\mathbf{x})$, is given directly by the neural network output. That is,
\begin{equation*}
    \dot{\mathbf{x}}  = f(\mathbf{x}) \approx \hat{f}_\theta(\mathbf{x}) = \mathcal{NN}_\theta(\mathbf{x}).
\end{equation*}
%where $f$ denotes the true underlying dynamics and $\mathcal{NN}_\theta$ is a neural network with parameters $\theta$. 
A trajectory is then obtained by numerical integration:
\begin{equation}
\begin{aligned}
    \hat{\mathbf{x}}_i &= \mathrm{ODESolve} (\hat{f}_\theta, \mathbf{x}(t_0), t_0, t_i) \\
    &\approx \mathbf{x}(t_0) + \int_{t_0}^{t_i} \hat{f}_\theta(\hat{\mathbf{x}}(\tau))\,\mathrm{d} \tau \\
    &\approx \mathbf{x}(t_i).
\end{aligned}
    \label{eq:neuralODE_general}
\end{equation}
Here, the first approximation is due to the error in the discretization method used by ODESolve, while the second is due to the approximation error of $\hat{f}_\theta$ to $f$. Notice that the numerical integrator must be explicit if we are to avoid solving a non-linear system of equations with a root-finding algorithm at each training iteration. This is different from the training approach of HNNs mentioned above, where we can use mono-implicit methods, which makes it possible to use symplectic integrators during training, also for non-separable systems \cite{Celledoni2025LearningIntegrators, Eidnes2023Pseudo-HamiltonianForces}. This is because HNN training loss is done on the integration from one data point to the next, given by \eqref{eq:hnn_training}, while the neural ODE training loss relies on integration from the initial state to all future states of a given trajectory:
\begin{equation}
\begin{aligned}
    \mathcal{L}(\theta) &= \frac{1}{N+1}\sum_{i=0}^{N} \|\mathbf{x}(t_i) - \hat{\mathbf{x}}_i\|_2^2 \\
    &= \frac{1}{N+1}\sum_{i=1}^{N} \left\| \mathbf{x}(t_i) - \mathrm{ODESolve} (\hat{f}_\theta, \mathbf{x}(t_0), t_0, t_i)\right \|_2^2 \\
    &= \frac{1}{N+1}\sum_{i=1}^{N} \left\| \mathbf{x}(t_i) - \mathrm{ODESolve} (\hat{f}_\theta, \hat{\mathbf{x}}_{i-1}, t_{i-1}, t_i)\right \|_2^2.
    \label{eq:neuralODE_training}
\end{aligned}
\end{equation}
Note that the neural ODE framework can naturally handle unevenly or sparsely sampled training data, since the formulation imposes no restrictions on the sampling times. 

In our neural Hamiltonian ODE (NHODE) framework, we assume a canonical Hamiltonian system and indirectly learn the equations of motion through learning the Hamiltonian, by $f_\theta = S \nabla \mathcal{H}_\theta$, but we use a neural-ODE-type training loss \ref{eq:neuralODE_training}. This makes it possible to learn models of partially observed Hamiltonian systems, where the HNN approach will not be possible and the general neural ODE approach will struggle to learn accurate models. Imposing Hamiltonian dynamics restricts the solution space of the model to only learn dynamics that are consistent with the underlying physics, compensating for lacking observations. The NHODE approach lets us learn systems where some state variables are completely unobserved, by explicitly excluding them from the loss function. By splitting up the state $\mathbf{x} = \left(\mathbf{x}_{\mathrm{obs}}, \mathbf{x}_{\mathrm{hid}} \right)^{\mathsf{T}}$ in an observed part $\mathbf{x}_{\mathrm{obs}}$ and a latent part $\mathbf{x}_{\mathrm{hid}}$, we define the loss function 
\begin{equation*}
\begin{aligned}
    \mathcal{L}(\theta) 
    &= \frac{1}{N+1}\sum_{i=0}^N \|\mathbf{x}_{\mathrm{obs}}(t_i) - \hat{\mathbf{x}}_{\mathrm{obs},i}\|_2^2\\
    &= \frac{1}{N+1}\sum_{i=0}^N \left\| \mathbf{x}_{\mathrm{obs}}(t_i) - \left[ \mathrm{ODESolve}(S\nabla \hat{\mathcal{H}}_\theta,\left(\mathbf{x}_{\mathrm{obs}}(0), \mathbf{x}_{\mathrm{hid}}(0)\right)^{\mathsf{T}},t_0,t_{i})\right]_{\mathrm{obs}} \right \|_2^2,
\end{aligned}
\end{equation*}
where $[\cdot]_{\mathrm{obs}}$ describes the observed part of the state vector only. 

This training approach also allows for learning the initial conditions of the unobserved variables $\mathbf{x}_{\mathrm{hid}}(0)$. To that end, we introduce a separate neural network as an encoder to learn the hidden initial state $\mathbf{x}_{\mathrm{hid}} (0)$, given the first $K$ time steps of the observed variables:
\begin{equation*}
    \mathbf{x}_{\mathrm{hid}} (0) \approx \hat{\mathbf{x}}_{\phi,\mathrm{hid},0} = \mathcal{NN}_{\phi}
    \left(
    \mathbf{x}_{\mathrm{obs}}(t_0),\ldots,\mathbf{x}_{\mathrm{obs}}(t_K)
    \right).
\end{equation*}
The two networks $\hat{\mathcal{H}}_\theta$ and $\hat{\mathbf{x}}_{\phi,\mathrm{hid},0}$ are then trained jointly using the same shared loss function
\begin{equation}
    \mathcal{L}(\theta, \phi) 
    = \frac{1}{N+1}\sum_{i=0}^N \left\| \mathbf{x}_{\mathrm{obs}}(t_i) - \left[ \mathrm{ODESolve}(S\nabla \hat{\mathcal{H}}_\theta,\left(\mathbf{x}_{\mathrm{obs}}(0), \hat{\mathbf{x}}_{\phi,\mathrm{hid},0}\right)^{\mathsf{T}},t_0,t_{i})\right]_{\mathrm{obs}} \right \|_2^2.
    \label{eq:full_loss_function}
\end{equation}
Here, it is also possible to incorporate system-specific constraints in the encoder network to ensure that the predicted initial conditions are physically consistent, e.g. by enforcing positivity of the initial position. 

\subsection{Coordinate transformation and known kinetic energy} 
\label{sec:incorporating-physics}

While all NHODE models preserve the learned Hamiltonian, further physics-informed assumptions are imposed to obtain the models we call NHODE$_{\textrm{tot,rel}}$, NHODE$_{\textrm{pot,abs}}$ and NHODE$_{\textrm{pot,rel}}$; see Table \ref{tab:overview_phys_quantities_conserved_by_models}. 

For the NHODE$_{\textrm{tot,rel}}$ model, we apply a coordinate transformation in the systems where we assume translational symmetry. Translational symmetry here means that if the positions $\mathbf{q}_i$ of all point masses $i$ in the system are shifted by the same constant factor, the Hamiltonian remains unchanged. That is, we assume the Hamiltonian can be expressed as $\mathcal{H}(d(\mathbf{q}), \mathbf{p})$, where $d(\mathbf{q}) = (\| \mathbf{q}_j - \mathbf{q}_k \|)_{j<k}$. Thus, we set up a neural network $\hat{\mathcal{H}}_\theta(\mathbf{q}, \mathbf{p}) = \tilde{\mathcal{H}}_\theta(d(\mathbf{q}), \mathbf{p})$ that is given relative distances instead of absolute positions as input. Importantly, the gradients are still taken with respect to the original state coordinates, backpropagating the loss through the coordinate transformation. This ensures that the total energy is still conserved, and in addition, the translational symmetry implies conservation of linear momentum; see Appendix \ref{sec:proofs} for a proof.

If we assume that the Hamiltonian is separable into potential and kinetic energy, $\mathcal{H}(\mathbf{q}, \mathbf{p}) = \mathcal{T}(\mathbf{p}) + \mathcal{V}(\mathbf{q})$, and we further assume that the kinetic energy is in the standard form, we may learn an approximation of the Hamiltonian by only learning the potential energy depending on positions:
\begin{equation*}
    \hat{\mathcal{H}}_\theta(\mathbf{q}, \mathbf{p}) = \sum_{j=1}^{N_p} \frac{ \| \mathbf{p}_j \|^2}{2m_j} + \hat{\mathcal{V}}_\theta (\mathbf{q}),
    %\label{eq:hamiltonian_known_kinetic_energy}
\end{equation*}
which leads to what we call the NHODE$_{\textrm{pot,abs}}$ model. If in addition translational symmetry is assumed, we set $\hat{\mathcal{V}}_\theta (\mathbf{q}) = \tilde{\mathcal{V}}_\theta (d(\mathbf{q}))$, to obtain the NHODE$_{\textrm{pot,rel}}$ model. This model guarantees conservation of angular momentum, in addition to linear momentum and the learned Hamiltonian. A proof is provided in Appendix \ref{sec:proofs}.

\section{Code availability}

The code is available on the GitHub repository \url{https://github.com/sunnivameltzer/nhode}, and has also been archived on \cite{Meltzer2026Neuralequations}.

\section{Data availability}

Training data are generated by running the code with the seeds given in the configuration files for each experiment. For reproducibility, all the trained models used to generate the results in this paper are also made available on \cite{Meltzer2026Neuralequations}.

\section{Acknowledgments}

All authors acknowledge support from the Research Council of Norway, through the projects MMSIML (project no. 346003) and PhysML (project no. 338779).

\section{Author contributions}

S.M.\ contributed to the conceptualization, methodology, investigation, software, visualization and writing (original draft). S.E.\ and A.J.S.\ contributed to the conceptualization, methodology, writing (original draft) and supervision. All authors reviewed and edited the manuscript.

\section{Competing interests}

The authors declare no competing interests.

\appendix

\section{Conservation properties of the learned system}\label{sec:proofs}

\begin{proposition}[Conservation of the learned Hamiltonian]
Any differentiable function $\hat{\mathcal{H}}_\theta : \mathbb{R}^{2n} \to \mathbb{R}$ is conserved along solutions of the Hamiltonian system \eqref{eq:hamilton} with $\mathcal{H} = \hat{\mathcal{H}}_\theta$.
\end{proposition}

\begin{proof}
Along a solution $\mathbf{x}(t)$ of \eqref{eq:hamilton}, we have
\begin{equation*}
    \frac{d}{dt}\hat{\mathcal{H}}_\theta(\mathbf{x}(t))
    =
    \nabla \hat{\mathcal{H}}_\theta(\mathbf{x}(t))^\mathsf{T}\dot{\mathbf{x}}(t)
    =
    \nabla \hat{\mathcal{H}}_\theta(\mathbf{x}(t))^\mathsf{T}
    S
    \nabla \hat{\mathcal{H}}_\theta(\mathbf{x}(t))
    =
    0.
\end{equation*}
The last equality follows from $S$ being skew-symmetric, meaning that $\mathbf{v}^\mathsf{T}S\mathbf{v}=0$ for any $\mathbf{v} \in \mathbb{R}^{2n}$.
\end{proof}

\begin{proposition}[Conservation of linear momentum]
Let $\mathbf{q} = (\mathbf{q}_1, \dots, \mathbf{q}_{N_p})$, where $\mathbf{q}_j \in \mathbb{R}^d$, and define the pairwise distance map
\begin{equation*}
    d(\mathbf{q}) = (\| \mathbf{q}_j - \mathbf{q}_k \|)_{j<k}.
\end{equation*}
If the learned Hamiltonian is a function of pairwise distances,
\begin{equation*}
    \hat{\mathcal{H}}_\theta(\mathbf{q}, \mathbf{p}) = \tilde{\mathcal{H}}_\theta(d(\mathbf{q}), \mathbf{p}),
\end{equation*}
then the total linear momentum
\begin{equation*}
    \mathbf{P} = \sum_{j=1}^{N_p} \mathbf{p}_j
\end{equation*}
is conserved along solutions of the Hamiltonian system \eqref{eq:Hamiltons_equations} with $\mathcal{H} = \hat{\mathcal{H}}_\theta$.
\end{proposition}

\begin{proof}
For any global translation $\mathbf{a} \in \mathbb{R}^d$,
\begin{equation*}
    (\mathbf{q}_j+\mathbf{a})-(\mathbf{q}_k+\mathbf{a}) = \mathbf{q}_j-\mathbf{q}_k,
\end{equation*}
and hence $d(\mathbf{q}+\mathbf{a})=d(\mathbf{q})$. Therefore
\begin{equation*}
    \hat{\mathcal{H}}_\theta(\mathbf{q}+\mathbf{a},\mathbf{p})
    = \tilde{\mathcal{H}}_\theta(d(\mathbf{q}+\mathbf{a}),\mathbf{p})
    = \tilde{\mathcal{H}}_\theta(d(\mathbf{q}),\mathbf{p})
    = \hat{\mathcal{H}}_\theta(\mathbf{q},\mathbf{p}),
\end{equation*}
so $\hat{\mathcal{H}}_\theta$ is translation invariant. Differentiating this identity with respect to $\mathbf{a}$ at $\mathbf{a}=\mathbf{0}$ gives
\begin{equation*}
    \mathbf{0}
    =
    \left.\nabla_{\mathbf{a}}\hat{\mathcal{H}}_\theta(\mathbf{q}_1+\mathbf{a},\dots,\mathbf{q}_{N_p}+\mathbf{a},\mathbf{p})\right|_{\mathbf{a}=\mathbf{0}}
    =
    \sum_{j=1}^{N_p}\frac{\partial \hat{\mathcal{H}}_\theta}{\partial \mathbf{q}_j}.
\end{equation*}
Using Hamilton's equations,
\begin{equation*}
    \frac{d\mathbf{P}}{dt}
    =
    \sum_{j=1}^{N_p} \frac{d\mathbf{p}_j}{dt}
    =
    -\sum_{j=1}^{N_p}\frac{\partial \hat{\mathcal{H}}_\theta}{\partial \mathbf{q}_j}
    =
    \mathbf{0}.
\end{equation*}
Hence, $\mathbf{P}$ is conserved.
\end{proof}

\begin{proposition}[Conservation of angular momentum]
Let $\mathbf{q} = (\mathbf{q}_1, \dots, \mathbf{q}_{N_p})$ and $\mathbf{p} = (\mathbf{p}_1, \dots, \mathbf{p}_{N_p})$, where $\mathbf{q}_j,\mathbf{p}_j \in \mathbb{R}^d$ with $d=2$ or $d=3$. Define
\begin{equation*}
    r_{jk} := \|\mathbf{q}_j - \mathbf{q}_k \|, \qquad d(\mathbf{q}) := (r_{jk})_{j<k}.
\end{equation*}
Assume the learned Hamiltonian is separable and of the form
\begin{equation}
    \hat{\mathcal{H}}_\theta(\mathbf{q}, \mathbf{p})
    =
    \sum_{j=1}^{N_p}\frac{\|\mathbf{p}_j\|^2}{2m_j}
    +
    \tilde{\mathcal{V}}_\theta(d(\mathbf{q})).
    \label{eq:hamiltonian_for_conservation_angular_momentum}
\end{equation}
Then the total angular momentum
\begin{equation*}
    \mathbf{L} = \sum_{j=1}^{N_p} \mathbf{q}_j \times \mathbf{p}_j
\end{equation*}
is conserved along solutions of \eqref{eq:Hamiltons_equations} with $\mathcal{H} = \hat{\mathcal{H}}_\theta$. For $d=2$, the cross product is understood by embedding planar vectors in $\mathbb{R}^3$, so that $\mathbf{L}$ has only an out-of-plane component.
\end{proposition}

\begin{proof}
We give the proof using the three-dimensional cross product; the planar case follows by embedding $\mathbb{R}^2$ in $\mathbb{R}^3$. Differentiating the total angular momentum gives
\begin{equation}
\begin{aligned}
    \frac{d\mathbf{L}}{dt}
    &=
    \sum_{j=1}^{N_p}\left(\dot{\mathbf{q}}_j \times \mathbf{p}_j + \mathbf{q}_j \times \dot{\mathbf{p}}_j\right) \\
    &=
    \sum_{j=1}^{N_p}\left(
        \frac{\partial \hat{\mathcal{H}}_\theta}{\partial \mathbf{p}_j} \times \mathbf{p}_j
        -
        \mathbf{q}_j \times \frac{\partial \hat{\mathcal{H}}_\theta}{\partial \mathbf{q}_j}
    \right),
\end{aligned}
\label{eq:time_derivative_angular_momentum}
\end{equation}
where we used Hamilton's equations. The kinetic-energy term vanishes since
\begin{equation*}
    \frac{\partial \hat{\mathcal{H}}_\theta}{\partial \mathbf{p}_j} \times \mathbf{p}_j
    =
    \frac{\mathbf{p}_j}{m_j} \times \mathbf{p}_j
    =
    \mathbf{0}.
\end{equation*}
Since the potential depends on $\mathbf{q}$ only through pairwise distances, the chain rule gives
\begin{equation*}
    \frac{\partial \hat{\mathcal{H}}_\theta}{\partial \mathbf{q}_j}
    =
    \frac{\partial \tilde{\mathcal{V}}_\theta}{\partial \mathbf{q}_j}
    =
    \sum_{k\neq j} f_{jk}(\mathbf{q}_j-\mathbf{q}_k),
    \qquad
    f_{jk}:=\frac{1}{r_{jk}}\frac{\partial \tilde{\mathcal{V}}_\theta}{\partial r_{jk}}.
\end{equation*}
Here, $\partial \tilde{\mathcal{V}}_\theta/\partial r_{jk}$ denotes the derivative with respect to the pairwise-distance component associated with the unordered pair $\{j,k\}$, so $f_{jk}=f_{kj}$. Substituting into \eqref{eq:time_derivative_angular_momentum} yields
\begin{equation*}
    \frac{d\mathbf{L}}{dt}
    =
    -\sum_{j=1}^{N_p}\sum_{k\neq j}\mathbf{q}_j \times f_{jk}(\mathbf{q}_j-\mathbf{q}_k).
\end{equation*}
The terms cancel pairwise. For each unordered pair $\{j,k\}$,
\begin{equation*}
\begin{aligned}
    &\mathbf{q}_j \times f_{jk}(\mathbf{q}_j-\mathbf{q}_k)
    +
    \mathbf{q}_k \times f_{kj}(\mathbf{q}_k-\mathbf{q}_j) \\
    &\qquad =
    f_{jk}\left[
        \mathbf{q}_j \times (\mathbf{q}_j-\mathbf{q}_k)
        +
        \mathbf{q}_k \times (\mathbf{q}_k-\mathbf{q}_j)
    \right] \\
    &\qquad =
    f_{jk}\left[
        \mathbf{q}_j \times \mathbf{q}_j
        -
        \mathbf{q}_j \times \mathbf{q}_k
        +
        \mathbf{q}_k \times \mathbf{q}_k
        -
        \mathbf{q}_k \times \mathbf{q}_j
    \right]
    =
    \mathbf{0}.
\end{aligned}
\end{equation*}
Thus all pairwise contributions vanish, and therefore $d\mathbf{L}/dt=\mathbf{0}$. Hence, $\mathbf{L}$ is conserved.
\end{proof}

\section{Testing optimal Plummer softening parameter}
\label{sec:plummer-parameter-test}

Figure \ref{fig:epsilon_range_test} shows the mean squared test error for ten models trained with different values of the Plummer softening parameter $\varepsilon$. The purpose of this test was to find a compromise between increased numerical stability and minimal modification of the original, unsoftened Hamiltonian. Based on these results, a value of $\varepsilon = 0.6$ was chosen for the three-body problem experiments.

\begin{figure}[tbh]
    \centering
    \includegraphics[width=0.8\linewidth]{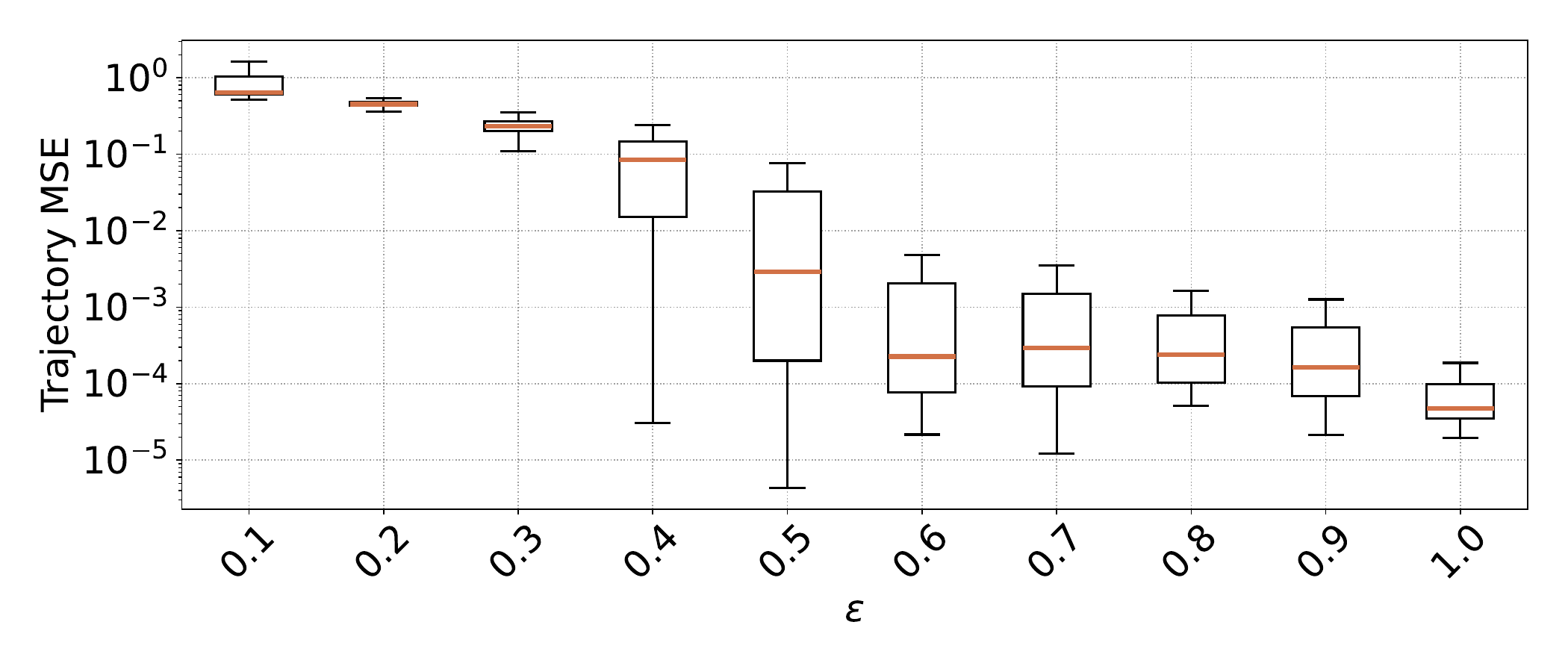}
    \caption{Testing the optimal value of the Plummer softening parameter in the three body problem. Ten different models were trained using different values of $\varepsilon$, and then tested on 128 trajectories from random initial conditions, where each test trajectory was rolled out 10000 steps using $\delta t=0.002$. The box plot shows the distribution of trajectory-wise test errors for each value of $\varepsilon$. Based on this, we chose $\varepsilon = 0.6$ for the three-body experiments.}
    \label{fig:epsilon_range_test}
\end{figure}

\section{Long-horizon rollouts three body problem}
\label{sec:super-long-rollouts-three-body-problem}

Figure \ref{fig:all-methods-superlong-rollout} shows predicted trajectories of the hidden point mass in the three body problem for a randomly selected test initial condition, rolled out for 20,000 time steps with $\Delta t = 0.01$. Although we don't expect pointwise agreement between the predicted trajectories and the ground truth over such long time horizons, these rollouts provide a qualitative comparison of model stability in the chaotic regime. The NHODE$_{\textrm{pot,rel}}$ model remains stable and reproduces the qualitative structure of the true dynamics, whereas the trajectories predicted by the other models become unstable and diverge.

\begin{figure}[tbh]
    \centering
    \includegraphics[width=\linewidth]{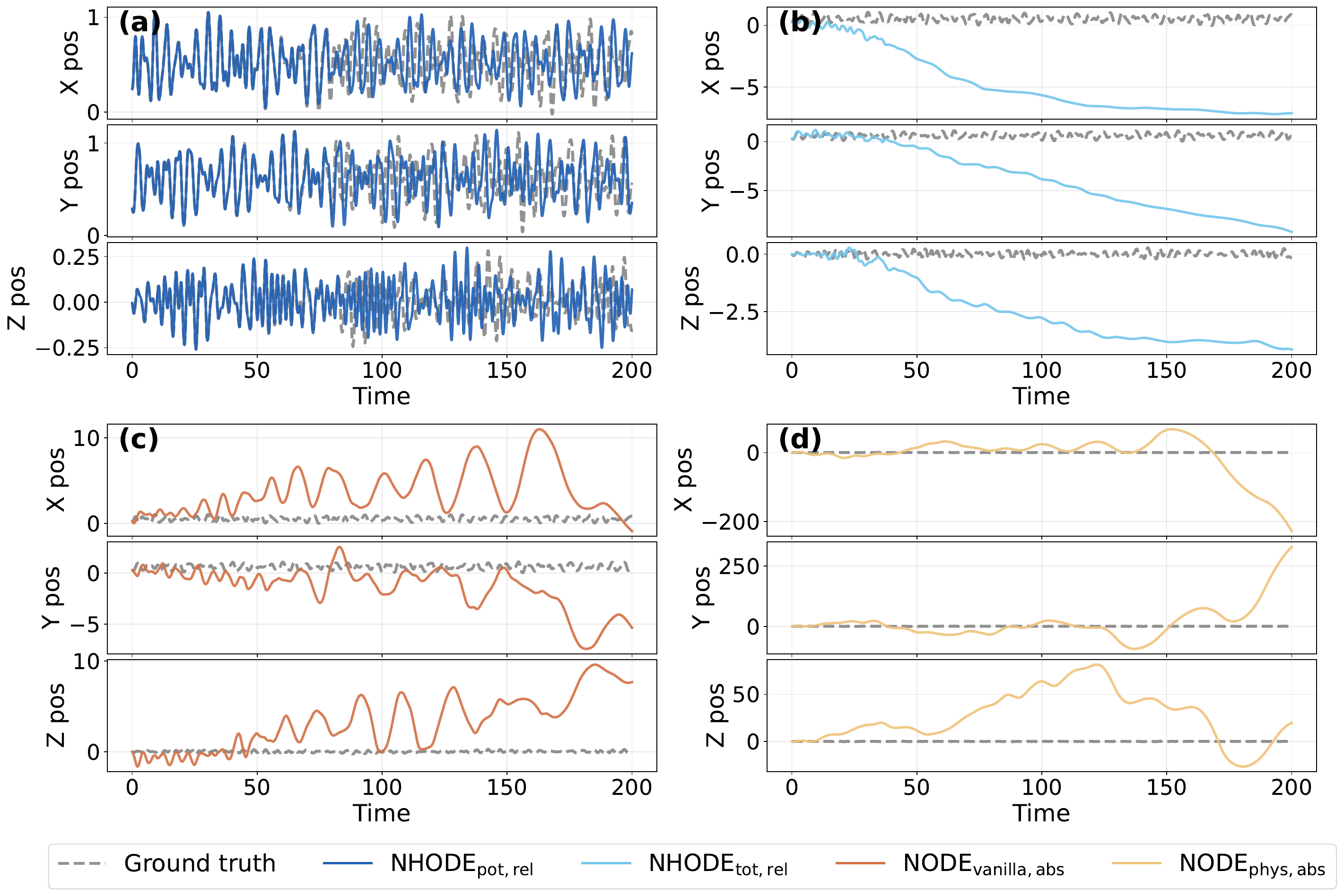}
    \caption{Long-horizon rollouts for the three-body problem, showing the predicted position of the unobserved point mass $m_1$. For each method, we show the predictions from one representative model rather than averaged trajectories across independently trained models, as trajectory averaging is not meaningful in the chaotic regime. Predicted and true trajectories are compared coordinate-wise in Cartesian coordinates $(x,y,z)$. Panels (a--d) correspond to the four models NHODE$_{\textrm{pot,rel}}$, NHODE$_{\textrm{tot,rel}}$, NODE$_{\textrm{vanilla,abs}}$, and NODE$_{\textrm{phys,abs}}$, respectively.
    }
    \label{fig:all-methods-superlong-rollout}
\end{figure}

\section{Experimental Details}
\label{sec:experimental-details}

All model types are implemented as $\tanh$-activated MLPs with a hidden dimension of 128 and a linear output layer, implemented in JAX~\cite{Bradbury2018JAX:Programs} using the Equinox~\cite{Kidger2021Equinox:Transformations} library. The model depths are given in table \ref{tab:hyperparams}.

The models are trained to minimise the MSE on the observed state dimensions only (equation \ref{eq:full_loss_function}). Training data are generated by integrating the ground-truth dynamics from randomly sampled initial conditions, with velocities converted to canonical momenta $p = mv$. 
%Data are split 85\%/15\% (train/validation) by random permutation. 
85\% of the training data are used for training, selected by random permutation.
All models are trained with the Adam optimiser~\cite{Kingma2017Adam:Optimization} using an initial learning rate of $\eta = 10^{-3}$, mini-batch size 64, and an exponential learning rate decay schedule with decay rate 0.95 applied every 10 epochs, for a total of 1200 epochs. Ten independent random seeds are used per model and system.

During training, all trajectories are integrated with the Tsit5 solver~\cite{Tsitouras2011RungeKuttaAssumption} from \texttt{diffrax}~\cite{Kidger2021OnEquations}, using an adaptive PID step-size controller with $\delta t_0$ chosen automatically.

The trained models are evaluated on 10 unseen initial conditions sampled randomly from the same domain as the training data, and rolled out over $t \in [0, 20]$ using 2001 time steps ($\delta t = 0.01$).

\label{app:systems}

\begin{table}[h]
\centering
\caption{System-specific parameters. Training data is identical across all model types within a system. $^\dagger$The three-body problem uses looser solver tolerances and a maximum step limit to accommodate chaotic dynamics.}
\label{tab:hyperparams}
\begin{tabular}{lcccc}
\toprule
 & \textbf{DMS} & \textbf{NLMS} & \textbf{3BP} & \textbf{DMS-IC} \\
\midrule
State dimension             & 4   & 12  & 18  & 4   \\
Model depth                 & 4   & 3   & 4   & 4   \\
Number of training samples  & 4000 & 4000 & 2000 & 4000 \\
Trajectory length           & 1.0 & 1.0 & 3.0 & 1.0 \\
Steps per trajectory        & 101 & 101 & 1001 & 101 \\
rtol / atol                 & $10^{-6}$ / $10^{-8}$ & $10^{-6}$ / $10^{-8}$ & $10^{-5}$ / $10^{-6}$$^\dagger$ & $10^{-6}$ / $10^{-8}$ \\
Max solver steps            & ---  & ---  & 16\,384$^\dagger$ & ---  \\
\bottomrule
\end{tabular}
\end{table}

\printbibliography

@inproceedings{Chu1991ASystems,
    author={Chu, S. Reynold and Shoureshi, Rahmat},
    booktitle={1991 American Control Conference}, 
    title={A Neural Network Approach for Identification of Continuous-Time Nonlinear Dynamic Systems}, 
    year={1991},
    volume={},
    number={},
    pages={1-5},
    doi={10.23919/ACC.1991.4791308}
}

@article{E2017ASystems,
    title = {A Proposal on Machine Learning via Dynamical Systems},
    year = {2017},
    journal = {Communications in Mathematics and Statistics},
    author = {E, Weinan},
    number = {1},
    month = {3},
    pages = {1--11},
    volume = {5},
    publisher = {Springer Verlag},
    doi = {10.1007/s40304-017-0103-z},
    issn = {2194671X}
}

@article{Kingma2017Adam:Optimization,
    author = {Kingma, Diederik and Ba, Jimmy},
    year = {2014},
    month = {12},
    pages = {},
    title = {Adam: A Method for Stochastic Optimization},
    journal = {International Conference on Learning Representations}
}

@article{Lagaris1998ArtificialEquations,
    title = {Artificial neural networks for solving ordinary and partial differential equations},
    year = {1998},
    journal = {IEEE Transactions on Neural Networks},
    publisher={Institute of Electrical and Electronics Engineers (IEEE)},
    author = {Lagaris, Isaac and Likas, Aristidis and Fotiadis, Dimitrios},
    number = {5},
    pages = {987--1000},
    volume = {9},
    doi = {10.1109/72.712178},
    issn = {10459227}
}

@misc{Dandekar2022BayesianEquations,
    title = {Bayesian Neural Ordinary Differential Equations},
    year = {2022},
    author = {Dandekar, Raj and Chung, Karen and Dixit, Vaibhav and Tarek, Mohamed and Garcia-Valadez, Aslan and Vemula, Krishna Vishal and Rackauckas, Chris},
    month = {2},
    eprint={2012.07244}, 
    archivePrefix={arXiv},
    %url={https://arxiv.org/abs/2012.07244}
}

@misc{Zhu2020DeepIntegrators,
    title = {Deep Hamiltonian networks based on symplectic integrators},
    year = {2020},
    author = {Zhu, Aiqing and Jin, Pengzhan and Tang, Yifa},
    month = {4},
    eprint={2004.13830},
    archivePrefix={arXiv},
    primaryClass={math.NA},
    %url={https://arxiv.org/abs/2004.13830},
}

@misc{Rackauckas2019DiffEqFlux.jlEquations,
    title = {DiffEqFlux.jl - A Julia Library for Neural Differential Equations},
    year = {2019},
    author = {Rackauckas, Chris and Innes, Mike and Ma, Yingbo and Bettencourt, Jesse and White, Lyndon and Dixit, Vaibhav},
    month = {2},
    eprint={1902.02376},
    archivePrefix={arXiv},
    primaryClass={cs.LG},
    %url={https://arxiv.org/abs/1902.02376}
}

@inproceedings{Cranmer2020DiscoveringBiases,
     author = {Cranmer, Miles and Sanchez Gonzalez, Alvaro and Battaglia, Peter and Xu, Rui and Cranmer, Kyle and Spergel, David and Ho, Shirley},
    booktitle = {Advances in Neural Information Processing Systems},
    editor = {H. Larochelle and M. Ranzato and R. Hadsell and M.F. Balcan and H. Lin},
    pages = {17429--17442},
    publisher = {Curran Associates, Inc.},
    title = {Discovering Symbolic Models from Deep Learning with Inductive Biases},
    url = {https://proceedings.neurips.cc/paper_files/paper/2020/file/c9f2f917078bd2db12f23c3b413d9cba-Paper.pdf},
    volume = {33},
    year = {2020}
}

@article{Rico-Martinez1992Discrete-Data,
    title = {Discrete- Vs. Continuous-Time Nonlinear Signal Processing of Cu Electrodissolution Data},
    year = {1992},
    journal = {Chemical Engineering Communications},
    publisher = {Taylor \& Francis},
    author = {Rico-Mart{\'{i}}nez, R. and Krischer, K. and Kevrekidis, I. G. and Kube, M. C. and Hudson, J. L.},
    number = {1},
    pages = {25--48},
    volume = {118},
    doi = {10.1080/00986449208936084},
    issn = {15635201}
}

@article{Kidger2021Equinox:Transformations,
    author = {Kidger, Patrick and Garcia, Cristian},
    title={{E}quinox: neural networks in {JAX} via callable {P}y{T}rees and filtered transformations},
    year={2021},
    journal={Differentiable Programming workshop at Neural Information Processing Systems 2021},
    eprint={2111.00254},
    archivePrefix={arXiv},
}

@misc{Matsubara2023FINDE:Quantities,
    title = {FINDE: Neural Differential Equations for Finding and Preserving Invariant Quantities},
    year = {2023},
    author = {Matsubara, Takashi and Yaguchi, Takaharu},
    month = {3},
    %url = {http://arxiv.org/abs/2210.00272},
    eprint={2210.00272},
    archivePrefix={arXiv},
}

@inproceedings{Toth2020HamiltonianNetworks,
    title = {Hamiltonian Generative Networks},
    booktitle = {International Conference on Learning Representations},
    year = {2020},
    author = {Toth, Peter and Rezende, Danilo Jimenez and Jaegle, Andrew and Racani{\`{e}}re, Sébastien and Botev, Aleksandar and Higgins, Irina},
    month = {2},
    eprint={1909.13789},
    archivePrefix={arXiv},
}

@misc{Sanchez-Gonzalez2019HamiltonianIntegrators,
    title = {Hamiltonian Graph Networks with ODE Integrators},
    year = {2019},
    author = {Sanchez-Gonzalez, Alvaro and Bapst, Victor and Cranmer, Kyle and Battaglia, Peter},
    month = {9},
    eprint={1909.12790},
    archivePrefix={arXiv},
    %url={https://arxiv.org/abs/1909.12790}, 
}

@inproceedings{Greydanus2019HamiltonianNetworks,
    author = {Greydanus, Sam and Dzamba, Misko and Yosinski, Jason},
    booktitle = {Advances in Neural Information Processing Systems},
    editor = {H. Wallach and H. Larochelle and A. Beygelzimer and F. d\textquotesingle Alch\'{e}-Buc and E. Fox and R. Garnett},
    pages = {},
    publisher = {Curran Associates, Inc.},
    title = {Hamiltonian Neural Networks},
    url = {https://proceedings.neurips.cc/paper_files/paper/2019/file/26cd8ecadce0d4efd6cc8a8725cbd1f8-Paper.pdf},
    volume = {32},
    year = {2019}
}

@software{Bradbury2018JAX:Programs,
    title = {{JAX}: composable transformations of {P}ython+{N}um{P}y programs},
    year = {2018},
    author = {Bradbury, James and Frostig, Roy and Hawkins, Peter and Johnson, Matthew James and Katariya, Yash and Leary, Chris and Maclaurin, Dougal and Necula, George and Paszke, Adam and Vander{\{}P{\}}las, Jake and Wanderman-{\{}M{\}}ilne, Skye and Zhang, Qiao},
    url = {http://github.com/jax-ml/jax},
    version = {0.3.13},
}

@misc{Cranmer2020LagrangianNetworks,
    title = {Lagrangian Neural Networks},
    year = {2020},
    author = {Cranmer, Miles and Greydanus, Sam and Hoyer, Stephan and Battaglia, Peter and Spergel, David and Ho, Shirley},
    month = {7},
    eprint={2003.04630},
    archivePrefix={arXiv},
    primaryClass={cs.LG},
    %url={https://arxiv.org/abs/2003.04630},
}

@article{Celledoni2025LearningIntegrators,
    title = {Learning dynamical systems from noisy data with inverse-explicit integrators},
    year = {2025},
    journal = {Physica D: Nonlinear Phenomena},
    author = {Celledoni, Elena and Eidnes, Sølve and Myhr, Håkon Noren},
    month = {2},
    pages = {134471},
    volume = {472},
    doi = {10.1016/j.physd.2024.134471},
    issn = {0167-2789}
}

@article{Grigorian2025LearningSystems,
    title = {Learning governing equations of unobserved states in dynamical systems},
    year = {2025},
    journal = {Physica D: Nonlinear Phenomena},
    author = {Grigorian, Gevik and George, Sandip V. and Arridge, Simon},
    month = {2},
    pages = {134499},
    volume = {472},
    doi = {10.1016/j.physd.2024.134499},
    issn = {0167-2789}
}

@article{Ghanem2024LearningMeasurements,
    title={Learning Physics Informed Neural ODEs with Partial Measurements}, 
    volume={39}, 
    %url={https://ojs.aaai.org/index.php/AAAI/article/view/33846}, 
    doi={10.1609/aaai.v39i16.33846}, 
    number={16}, 
    journal={Proceedings of the AAAI Conference on Artificial Intelligence}, 
    author={Ghanem, Paul and Demirkaya, Ahmet and Imbiriba, Tales and Ramezani, Alireza and Danziger, Zachary and Erdogmus, Deniz}, 
    year={2025}, 
    month={4}, 
    pages={16799–16807} 
}

@inproceedings{Chen2018NeuralEquations,
    author = {Chen, Ricky T. Q. and Rubanova, Yulia and Bettencourt, Jesse and Duvenaud, David K},
    booktitle = {Advances in Neural Information Processing Systems},
    editor = {S. Bengio and H. Wallach and H. Larochelle and K. Grauman and N. Cesa-Bianchi and R. Garnett},
    pages = {},
    publisher = {Curran Associates, Inc.},
    title = {Neural Ordinary Differential Equations},
    url = {https://proceedings.neurips.cc/paper_files/paper/2018/file/69386f6bb1dfed68692a24c8686939b9-Paper.pdf},
    volume = {31},
    year = {2018}
}

@inproceedings{Rahman2022NeuralIdentification,
    author={Rahman, Aowabin and Drgoňa, Ján and Tuor, Aaron and Strube, Jan},
    booktitle={2022 American Control Conference (ACC)}, 
    title={Neural Ordinary Differential Equations for Nonlinear System Identification}, 
    year={2022},
    volume={},
    number={},
    pages={3979-3984},
    doi={10.23919/ACC53348.2022.9867586}
}

@article{Bertalan2019OnData,
    title = {On learning Hamiltonian systems from data},
    year = {2019},
    journal = {Chaos: An Interdisciplinary Journal of Nonlinear Science},
    author = {Bertalan, Tom and Dietrich, Felix and Mezi{\'{c}}, Igor and Kevrekidis, Ioannis G.},
    number = {12},
    month = {12},
    pages = {121107},
    volume = {29},
    doi = {10.1063/1.5128231},
    issn = {1054-1500},
}

@phdthesis{Kidger2021OnEquations,
    title={{O}n {N}eural {D}ifferential {E}quations},
    year={2021},
    author = {Kidger, Patrick},
    school={University of Oxford},
    eprint={2202.02435},
    archivePrefix={arXiv},
}

@article{Karniadakis2021Physics-informedLearning,
    title = {Physics-informed machine learning},
    year = {2021},
    journal = {Nature Reviews Physics},
    author = {Karniadakis, George Em and Kevrekidis, Ioannis G. and Lu, Lu and Perdikaris, Paris and Wang, Sifan and Yang, Liu},
    number = {6},
    month = {6},
    day = {1},
    pages = {422--440},
    volume = {3},
    doi = {10.1038/s42254-021-00314-5},
    issn = {2522-5820}
}

@article{Raissi2019Physics-informedEquations,
    title = {Physics-informed neural networks: A deep learning framework for solving forward and inverse problems involving nonlinear partial differential equations},
    year = {2019},
    journal = {Journal of Computational Physics},
    author = {Raissi, M. and Perdikaris, P. and Karniadakis, G. E.},
    month = {2},
    pages = {686-707},
    volume = {378},
    publisher = {Academic Press Inc.},
    doi = {10.1016/j.jcp.2018.10.045},
    issn = {0021-9991}
}

@article{Desai2021Port-HamiltonianSystems,
    title = {Port-Hamiltonian Neural Networks for Learning Explicit Time-Dependent Dynamical Systems},
    year = {2021},
    author = {Desai, Shaan A. and Mattheakis, Marios and Sondak, David and Protopapas, Pavlos and Roberts, Stephen J.},
    journal = {Phys. Rev. E},
    volume = {104},
    issue = {3},
    pages = {034312},
    numpages = {10},
    year = {2021},
    month = {9},
    publisher = {American Physical Society},
    doi = {10.1103/PhysRevE.104.034312},
}

@article{Eidnes2023Pseudo-HamiltonianForces,
    title = {Pseudo-Hamiltonian neural networks with state-dependent external forces},
    year = {2023},
    journal = {Physica D: Nonlinear Phenomena},
    author = {Eidnes, Sølve and Stasik, Alexander J. and Sterud, Camilla and B{\o}hn, Eivind and Riemer-S{\o}rensen, Signe},
    month = {4},
    volume = {446},
    pages = {133673},
    publisher = {Elsevier B.V.},
    doi = {10.1016/j.physd.2023.133673},
    issn = {0167-2789},
    arxivId = {2206.02660}
}

@misc{Buisson-Fenet2023RecognitionODEs,
    title = {Recognition Models to Learn Dynamics from Partial Observations with Neural ODEs},
    year = {2023},
    author = {Buisson-Fenet, Mona and Morgenthaler, Valery and Trimpe, Sebastian and Di Meglio, Florent},
    month = {1},
    eprint={2205.12550},
    archivePrefix={arXiv},
    primaryClass={eess.SY},
    %url={https://arxiv.org/abs/2205.12550}, 
}

@article{Tsitouras2011RungeKuttaAssumption,
    title = {Runge–Kutta pairs of order 5(4) satisfying only the first column simplifying assumption},
    year = {2011},
    journal = {Computers {\&} Mathematics with Applications},
    author = {Tsitouras, Ch.},
    number = {2},
    month = {7},
    pages = {770--775},
    volume = {62},
    doi = {10.1016/j.camwa.2011.06.002},
    issn = {0898-1221}
}

@inproceedings{Quaglino2020SNODE:Identification,
    title = {SNODE: Spectral Discretization of Neural ODEs for System Identification},
    booktitle = {International Conference on Learning Representations},
    year = {2020},
    author = {Quaglino, Alessio and Gallieri, Marco and Masci, Jonathan and Koutn{\'{i}}k, Jan},
    eprint={1906.07038},
    archivePrefix={arXiv},
}

@article{Malani2023SomeInformation,
    title = {Some of the variables, some of the parameters, some of the times, with some physics known: Identification with partial information},
    year = {2023},
    journal = {Computers {\&} Chemical Engineering},
    author = {Malani, Saurabh and Bertalan, Tom S. and Cui, Tianqi and Avalos, José L. and Betenbaugh, Michael and Kevrekidis, Ioannis G.},
    month = {10},
    pages = {108343},
    volume = {178},
    doi = {10.1016/j.compchemeng.2023.108343},
    issn = {0098-1354}
}

@article{Haber2017StableNetworks,
    title = {Stable architectures for deep neural networks},
    year = {2017},
    journal = {Inverse Problems},
    author = {Haber, Eldad and Ruthotto, Lars},
    number = {1},
    pages = {014004},
    month = {12},
    volume = {34},
    publisher = {IOP Publishing},
    doi = {10.1088/1361-6420/aa9a90},
    issn = {13616420},
    arxivId = {1705.03341}
}

@article{Offen2022SymplecticSystems,
    title = {Symplectic integration of learned Hamiltonian systems},
    year = {2022},
    journal = {Chaos: An Interdisciplinary Journal of Nonlinear Science},
    author = {Offen, C. and Ober-Bl{\"{o}}baum, S.},
    number = {1},
    pages = {013122},
    month = {1},
    volume = {32},
    doi = {10.1063/5.0065913},
    issn = {1054-1500}
}

@article{David2023SymplecticNetworks,
    title = {Symplectic learning for Hamiltonian neural networks},
    year = {2023},
    journal = {Journal of Computational Physics},
    author = {David, Marco and M{\'{e}}hats, Florian},
    month = {12},
    pages = {112495},
    volume = {494},
    doi = {10.1016/j.jcp.2023.112495},
    issn = {0021-9991}
}

@inproceedings{Zhong2024SymplecticControl,
    title = {Symplectic ODE-Net: Learning Hamiltonian Dynamics with Control},
    booktitle = {International Conference on Learning Representations},
    year = {2020},
    author = {Zhong, Yaofeng Desmond and Dey, Biswadip and Chakraborty, Amit},
    eprint={1909.12077},
    archivePrefix={arXiv},
}

@inproceedings{Chen2020SymplecticNetworks,
    title = {Symplectic Recurrent Neural Networks},
    booktitle = {International Conference on Learning Representations},
    year = {2020},
    author = {Chen, Zhengdao and Zhang, Jianyu and Arjovsky, Martin and Bottou, Léon},
    eprint={1909.13334},
    archivePrefix={arXiv},
}

@misc{Rackauckas2021UniversalLearning,
    title = {Universal Differential Equations for Scientific Machine Learning},
    year = {2021},
    author = {Rackauckas, Christopher and Ma, Yingbo and Martensen, Julius and Warner, Collin and Zubov, Kirill and Supekar, Rohit and Skinner, Dominic and Ramadhan, Ali and Edelman, Alan},
    month = {11},
    eprint={2001.04385},
    archivePrefix={arXiv},
    primaryClass={cs.LG},
    %url={https://arxiv.org/abs/2001.04385}, 
}

@inproceedings{Botev2021WhichDynamics,
    title = {Which priors matter? Benchmarking models for learning latent dynamics},
    booktitle = {Thirty-fifth Conference on Neural Information Processing Systems Datasets and Benchmarks Track},
    year = {2021},
    author = {Botev, Aleksandar and London, Deepmind and Jaegle, Andrew and Wirnsberger, Peter and Hennes, Daniel and Higgins, Irina},
    eprint={2111.05458},
    archivePrefix={arXiv},
}

@article{Aarseth1963DynamicalGalaxies,
    author = {Aarseth, S. J. and Hoyle, F.},
    title = {Dynamical Evolution of Clusters of Galaxies, I},
    journal = {Monthly Notices of the Royal Astronomical Society},
    volume = {126},
    number = {3},
    pages = {223-255},
    year = {1963},
    month = {06},
    issn = {0035-8711},
    doi = {10.1093/mnras/126.3.223},
}

@article{Dehnen2001TowardsError,
    author = {Dehnen, Walter},
    title = {Towards optimal softening in three-dimensional N-body codes — I. Minimizing the force error},
    journal = {Monthly Notices of the Royal Astronomical Society},
    volume = {324},
    number = {2},
    pages = {273-291},
    year = {2001},
    month = {06},
    issn = {0035-8711},
    doi = {10.1046/j.1365-8711.2001.04237.x},
}

@software{Meltzer2026Neuralequations,
    title = {Neural Hamiltonian ordinary differential equations},
    year = {2026},
    author = {Meltzer, Sunniva and Eidnes, Sølve and Stasik, Alexander Johannes},
    doi = {10.5281/zenodo.20283041},
    version = {1.0},
}

\end{document}